\documentclass[pdflatex,sn-mathphys-num]{sn-jnl}% Math and Physical Sciences Numbered Reference Style
\usepackage{graphicx}%
\usepackage{multirow}%
\usepackage{amsmath,amssymb,amsfonts}%
\usepackage{amsthm}%
\usepackage{mathrsfs}%
\usepackage[title]{appendix}%
\usepackage{xcolor}%
\usepackage{textcomp}%
\usepackage{manyfoot}%
\usepackage{booktabs}%
\usepackage{algorithm}%
\usepackage{algorithmicx}%
\usepackage{algpseudocode}%
\usepackage{listings}%
\usepackage{ulem}
\usepackage{hyperref}
\usepackage{url}
\usepackage{amsmath}
\usepackage{amsthm}
\usepackage{amssymb}
\usepackage{longtable}
\usepackage{array}
\usepackage{booktabs}
\usepackage{arydshln}
\usepackage{svg}
\usepackage{tabularx}
\usepackage{multirow}
\usepackage{float}
\usepackage{bm}
\usepackage{subcaption}
\theoremstyle{thmstyleone}%
\newtheorem{theorem}{Theorem}%  meant for continuous numbers
\theoremstyle{thmstyletwo}%

\theoremstyle{thmstylethree}%

\begin{document}

\title[Article Title]{From Rational Answers to Emotional Resonance: The Role of Controllable Emotion Generation in Language Models}


\author[1,2]{\fnm{Yurui} \sur{Dong}}\email{yuruidong22@m.fudan.edu.cn}
\equalcont{These authors contributed equally to this work.}

\author[1]{\fnm{Luozhijie} \sur{Jin}}\email{lzjjin22@m.fudan.edu.cn}
\equalcont{These authors contributed equally to this work.}

\author[4]{\fnm{Yao} \sur{Yang}}\email{yaoyang@zhejianglab.com}
\author[4]{\fnm{Bingjie} \sur{Lu}}\email{bingjielu@zhejianglab.com}

\author*[4]{\fnm{Jiaxi} \sur{Yang}}\email{jiaxiyang@zhejianglab.com}
\author*[2,3]{\fnm{Zhi} \sur{Liu}}\email{zhiliu@njucm.edu.cn}

\affil[1]{\orgdiv{School of Data Science}, \orgname{Fudan University}, \orgaddress{\city{Shanghai}, \country{China}}}

\affil[2]{\orgdiv{School of Pharmacy}, \orgname{Nanjing University of Chinese Medicine}, \orgaddress{\city{Nanjing}, \country{China}}}

\affil[3]{\orgname{State Key Laboratory on Technologies for Chinese Medicine Pharmaceutical Process Control and Intelligent Manufacture}, 
\orgaddress{\city{Nanjing}, \country{China}}}

\affil[4]{\orgname{Zhejiang Lab}, \orgaddress{\city{Hangzhou}, \country{China}}}


\abstract{
\textbf{Purpose:} Emotion is a fundamental component of human communication, shaping understanding, trust, and engagement across domains such as education, healthcare, and mental health. While large language models (LLMs) exhibit strong reasoning and knowledge generation capabilities, they still struggle to express emotions in a consistent, controllable, and contextually appropriate manner. This limitation restricts their potential for authentic human–AI interaction.

\textbf{Methods:} We propose a controllable emotion generation framework based on \emph{Emotion Vectors (EVs)}—latent representations derived from internal activation shifts between neutral and emotion-conditioned responses. By injecting these vectors into the hidden states of pretrained LLMs during inference, our method enables fine-grained, continuous modulation of emotional tone without any additional training or architectural modification. We further provide theoretical analysis proving that EV steering enhances emotional expressivity while maintaining semantic fidelity and linguistic fluency.

\textbf{Results:} Extensive experiments across multiple LLM families show that the proposed approach achieves consistent emotional alignment, stable topic adherence, and controllable affect intensity. Compared with existing prompt-based and fine-tuning-based baselines, our method demonstrates superior flexibility and generalizability.

\textbf{Conclusion:} Emotion Vector (EV) steering provides an efficient and interpretable means of bridging rational reasoning and affective understanding in large language models, offering a promising direction for building emotionally resonant AI systems capable of more natural human–machine interaction.
}




\maketitle

\section{Introduction}
Emotion elevates us from mere information processors to beings with preferences, attachments, empathy, and the capacity to create meaning. It serves not merely as a tool, but as the fundamental means through which we experience life, comprehend the world, and establish profound connections with it. For instance, Education, healthcare, and mental health are among the most socially consequential domains of human life, where affective communication with various emotions is essential not only for solving practical problems but also for ensuring human well-being. In these domains, the role of emotion is deeply embedded. In education, learning outcomes are influenced not simply by the transmission of knowledge but by the emotional environment that sustains engagement and curiosity; a teacher’s encouragement or patience can significantly shape students’ motivation and persistence , which is show the effection of positive emotions~\citep{Guo2025, Shengyao2024}. In healthcare, extensive research has demonstrated that physicians’ emotional engagement and empathic communication improve patient adherence, satisfaction, and even clinical recovery trajectories~\citep{Derksen2013, Kim2004effect}. Likewise, in mental health and companionship settings, the capacity for emotional attunement is not an accessory but a prerequisite for meaningful support~\citep{sabour2022chatbotsmentalhealthsupport, Abargil02012025}. A counseling exchange devoid of empathy may satisfy informational needs yet leave the user’s deeper concerns unresolved. These examples underscore a critical principle: emotion is not ancillary to human-centered interaction but a central determinant of its effectiveness.

At the same time, the demand for such emotionally enriched interactions has been steadily growing with the development of modern society~\citep{B2023Evolving, BYRNE2024108063}. Yet Relying solely on humans to provide stable, high-quality affective support is becoming increasingly untenable: human providers are inherently variable in availability and emotional consistency~\citep{Garnett2023,healthcare11010056, JEON201921}, and the cost of delivering continuous, personalized care at scale is prohibitively high~\citep{Straat2023}, while these consistent emotional support is essential~\citep{Bailey2022, Brandao2025}. This challenge between rising expectations for emotional value and the practical limits of human resources calls for technological solutions that can complement and extend human capacity.

The advent of large language models (LLMs) has opened new opportunities for intelligent systems to enter these domains at scale. Applications range from AI tutors that adaptively guide students~\citep{Letourneau2025}, to clinical decision-support systems that assist physicians~\citep{LI20251, Lammert2024, li2025clicaregroundinglargelanguage}, to conversational agents offering mental health interventions or companionship~\citep{wu2024ilikesunniei, Li2023}. Early deployments have yielded promising results: LLM-powered tutoring has been shown to foster individualized learning experiences~\citep{Sharma2025}; clinical support tools can streamline patient communication and diagnostic reasoning~\citep{Goh2024}; and dialogue agents for psychological support have increased accessibility to low-cost interventions~\citep{He2023}. Such progress demonstrates the transformative potential of LLMs in addressing some of society’s most pressing challenges.

Yet, despite their successes, current LLMs are fundamentally limited in their ability to engage with users on an emotional level. Most models generate content that is either affect-neutral, inconsistent in tone, or uncontrolled in affective orientation~\citep{novikova2025consistencylanguagemodelscurrent}. This shortcoming has been documented across multiple application studies: educational chatbots often fail to sustain motivational discourse~\citep{info16030235}; medical assistants provide clinically accurate but emotionally detached responses; and mental health chatbots, while helpful in offering structured advice, lack the warmth, empathy, or reassurance characteristic of human counselors~\citep{naik2025artificialempathyaibased}. Moreover, existing reseaches have pointed out the concerns on the AI’s empathy, capabilities, safety, and human involvement in mental healthcare~\citep{Lee2024AI}.

Such deficiencies not only reduce user trust but also constrain the overall efficacy of these systems~\citep{roshanaei2025talklistenconnecthumans}. Indeed, a purely factual but emotionally sterile response often leaves users dissatisfied, much as a technically competent but emotionally indifferent teacher, doctor, or counselor would fail to meet human expectations~\citep{jozani2024roleemotionsinformationalsupport}.

The consequences of this limitation are nontrivial. Consider the classroom, where an AI tutor might successfully solve a mathematical problem but fail to offer encouragement to a struggling student, thereby missing an opportunity to build resilience and confidence. In clinical care, an AI assistant may inform a patient of treatment side effects accurately yet omit the reassurance that a physician might naturally provide to alleviate anxiety. In psychological support, a chatbot may suggest coping strategies but without the empathetic validation that reassures individuals of their worth and shared humanity. In each of these cases, the absence of emotional intelligence constrains the system’s ability to produce meaningful, lasting ~\citep{Lissak_2024}.

Aiming to evolve LLMs from problem-solving instruments into genuine human-centered collaborators, the capacity for controllable, consistent, and contextually appropriate emotional expression must be developed. Unlike the stochastic or incidental emotional cues sometimes present in current LLM outputs~\citep{abbasian2024empathymultimodalityconversationalinterfaces,10.19066/COGSCI.2024.35.1.002}, such emotion expression must be deliberate and adaptive, aligning with users’ affective states and situational demands. Crucially, this ability must not be hard-coded or manually scripted but rather systematically controllable within the generative process, enabling flexible, reliable, and self-consistent affective behavior. Achieving this capability is key to unlocking deeper integration of LLMs in education, healthcare, and mental health—domains where emotional attunement is as critical as cognitive competence.

To overcome these challenges, we propose an elegant but effective method for the controllable emotional and affective expressions LLMs. Our approach offers a universal solution that allows fine-grained control over the emotional tone and sentiment of generated text, without compromising its fluency or coherence. It only needs to extract the "\textbf{E}motion \textbf{V}ector" used by the LLM to express basic emotions with simple prompts. This EV is then applied during inference to guide and control the emotional qualities of the generated output. Comprehensive evaluations across a range of LLM architectures confirm the consistency and stability of the resulting emotional expressions. This demonstrated universality addresses a key limitation of previous approaches, which were often constrained to specific models or datasets. Ultimately, this work provides valuable insights into how to equip LLMs with a reliable capacity for generating contextually appropriate emotional responses with minimal computational overhead.

\section{Related works}
\paragraph{Emotional Representations and Dialouge Systems} To create agents or dialogue systems that simulate human expression, a significant body of research has focused on understanding and representing emotions as fundamental aspects of human communication \citep{qian2024thinktwicehumanliketwostage,xue2024echatemotionsensitivespokendialogue}. Various theoretical frameworks and computational models have been explored to capture the multifaceted nature of emotions, including categorical approaches~\citep{mohsin2019summarizing,ekman1971constants} (e.g., discrete emotions like joy, sadness, anger~\citep{ekman1992argument}) and dimensional approaches~\citep{russell1999core} (e.g., valence-arousal scales~\citep{russell1980circumplex}). These representations serve as the building blocks for infusing artificial systems with the ability to perceive, process, and generate emotional content, aiming to enable more natural and empathetic interactions with users. While\citet{zhou2018emotional} and \citet{2019Generating} proposed a way of \textbf{Emotion Embedding} to make the model "has" the emotion, where, models were forced to install a module to generate emotions. However, most methods are too complex or requires further training. To achieve an effective emotion system, it is essential for the model to have precise, quantifiable control over emotions, as well as a flexible, plug-and-play module that can be seamlessly integrated as needed. It should also be consistent along the whole dialog.

\paragraph{Instruct tuning and prompt based emotional control} A significant body of work has focused on leveraging fine-tuning or prompt techniques for LLMs. \citet{chen2023soulchat}, \citet{chen2024cause} and \citet{zheng2023buildingemotionalsupportchatbots} explored fine-tuning approaches to cultivate empathetic behavior in LLMs for psychological counseling and emotional supports. However, althrough instruct-tuning models have relatively good performance, they are often inflexible and struggle to adapt to a wide range of applications and model architectures, due to their predefined emotion categories or fixed sets of emotional datasets\cite{ghosh2024closer,liu2024emollms}. Moreover, prompting strategies have also been used to elicit emotions without model modification. \citet{li2024enhancing,wang2024negativepromptleveragingpsychologylarge,li2023largelanguagemodelsunderstand}However, prompting depends on elaborate templates and external evaluation modules to maintain effectiveness.

\paragraph{Inference-Time Vectors Editing} 
Recent studies have explored editing the internal representations of language models to achieve controlled generation\citet{dekoninck2023controlled,liu2024incontextvectorsmakingcontext,li2023inference}. They uses latent steering vectors that enable semantic or stylistic shifts by modifying hidden activations. However, while they can realize controlable generation, these methods mainly focuses on the last token position during extraction and lacks global significance\citet{todd2024functionvectorslargelanguage}. It is difficult to apply to tasks such as emotions that require high generalization. Most control vector-related work is sentence-level control\citet{subramani2022extracting}, and requires training, focusing only on regulating the model's output for a single sentence. There has not been much success in achieving global control, which is essential for tasks like emotion control. A good emotion control system should be global, as this is necessary for building an effective emotion system.

\paragraph{Our Position}
In contrast to the above paradigms, our method extracts reusable and efficient Emotion Vectors (EVs) by comparing model responses to emotion-inducing and neutral prompts. It is fully \textbf{unsupervised}, \textbf{highly robust and controllable}, requiring \textbf{no training} or architecture changes and is \textbf{global consistent}. EVs provide continuous and fine-grained control over emotional intensity through scalar scaling, enabling broad applicability across model families. Compared to previous approaches, EV offers a more general and efficient mechanism for emotion modulation in LLMs.

\section{Method}
\begin{figure}[h]
\centering
  \includegraphics[width=\columnwidth]{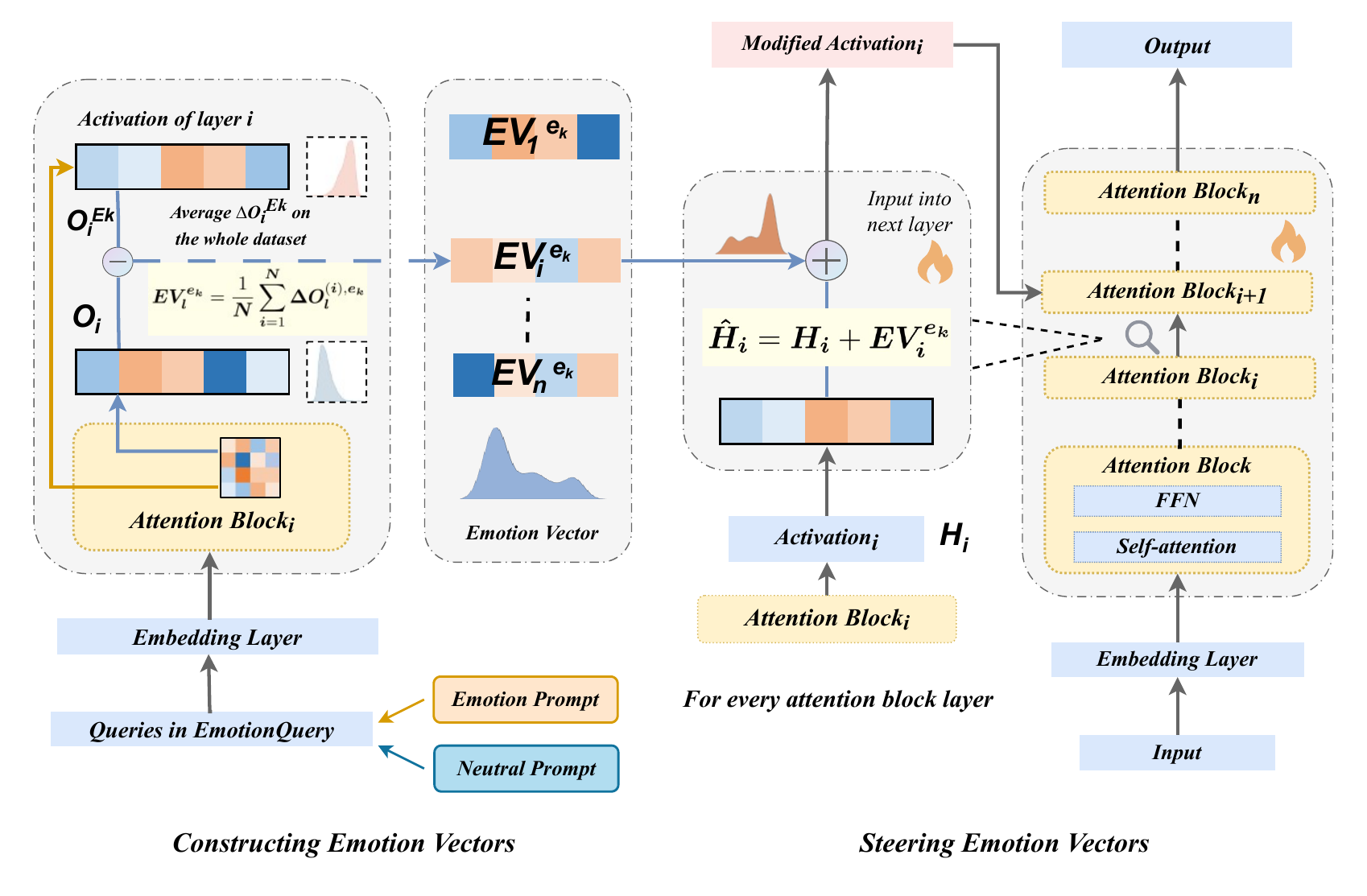}
  % \caption{The pipeline of how we steer the Emotion Vectors.}
  \caption{\textbf{Overview of the Emotion Vector (EV) pipeline.}
The figure follows the two-stage workflow used in our paper. 
\emph{EV extraction:} For each target emotion $e_k$, the model is run on \textit{EmotionQuery} with an emotion-conditioned prompt and a neutral prompt. 
At every transformer block $i$, we compute the token-averaged hidden outputs and take their difference $\Delta O_i^{(e_k)}$; averaging over $N$ queries yields a layer-wise vector $\mathrm{EV}_i^{(e_k)}=\tfrac{1}{N}\sum_{n=1}^{N}\Delta O_{i,n}^{(e_k)}$. 
\emph{EV steering at inference:} During generation, we inject the EV into the residual stream of \emph{every} attention block, modifying the hidden state as 
$\hat H_i = H_i + \alpha\,\mathrm{EV}_i^{(e_k)}$ (or $\alpha\,\mathrm{EV}^{\text{base}}$), and propagate the modified activations through subsequent self-attention/FFN blocks for each token. 
The scalar $\alpha$ provides continuous control of emotional intensity, and EVs can be combined additively if needed. 
This plug-and-play procedure leaves all model parameters frozen, yet steers the network toward the desired emotional direction while preserving semantic content.}

  \label{fig:pipeline}
\end{figure}
We propose a two-step method to identify and apply emotion vectors (EV) to guide the emotional tone of the language model's outputs. Emotion vectors (EVs) are added to the model’s internal representations without requiring additional training or changes to the model’s parameters. These vectors allow us to modulate the emotional tone of the output by steering the model's latent states, ensuring that the emotional direction is preserved while keeping the model's underlying parameters intact.

\subsection{Constructing Emotion Vectors}
To capture the emotional factors and semantics for LLM, a specialized dataset is designed and constructed to elicit specific emotional responses, referred to as \textit{EmotionQuery}. The dataset consists of 500 queries, with 100 queries generated for each of five emotional states derived from the basic emotion models\cite{ekman1992facial}: joy, anger, disgust, fear, and sadness to provoke the corresponding emotional reactions. The queries were generated by a GPT-4o-mini\cite{openai}. A more detailed description of the dataset and query construction process can be found in the Appendix~\ref{subsec:emotionqueries}.

Let’s denote the pretrained language model as \( \mathcal{M} \), which has \( L \) layers. The set of the five emotional states are denoted as \( E = \{e_1, e_2, \dots, e_K\} \), where \( e_k \) represents one emotion among the aforementioned 5 emontional states. For each query in \textit{EmotionQuery}, the model generates its responses under two settings:
\begin{itemize}
    \item A \textbf{neutral setting}, without emotional conditioning.
    \item An \textbf{emotional setting}, where the response reflects a specific emotion \( e_k \).
\end{itemize}

The goal of these generations is to measure how the model’s internal outputs change between these two settings and use these differences to define emotion vectors for each \( e_k \).

\paragraph{Capturing Internal Outputs.}
For each query, LLM generates the internal representations for its each layer, \( O_l \in \mathbb{R}^{T \times d} \) represent the output of the model at layer \( l \), where \( T \) is the number of output tokens corresponding to the input query, and \( d \) is the dimensionality of the hidden states. 

We compute the average of the outputs across all output tokens in the query:
\begin{equation}
\bar{O}_l = \frac{1}{T} \sum_{t=1}^{T} O_l[t],
\end{equation}
where \( \bar{O}_l \in \mathbb{R}^d \) represents the layer \( l \)’s aggregated output for the query, reducing token-level variability.

\paragraph{Measuring Emotional Shifts.}
For each query, the model generates averaged outputs \( \bar{O}_l \) under both the emotional and neutral settings. The difference between these outputs at layer \( l \) captures the shift caused by emotional conditioning for the emotion \( e_k \):
\begin{equation}
\Delta O_l^{e_k} = \bar{O}_l^{\text{emotion}(e_k)} - \bar{O}_l^{\text{neutral}},
\end{equation}
where \( \Delta O_l^{e_k} \in \mathbb{R}^d \) represents the emotional shift at layer \( l \) for the emotional state \( e_k \).

\paragraph{Constructing Emotion Vectors.}
To generalize the emotional shift across the dataset, we compute the average shift across all queries for a given emotional state \( e_k \). For each layer \( l \), the emotion vector is calculated as:
\begin{equation}
EV_l^{e_k} = \frac{1}{N} \sum_{i=1}^N \Delta O_l^{(i), e_k},
\end{equation}
where \( N \) is the number of queries for the emotional state \( e_k \), and \( EV_l^{e_k} \in \mathbb{R}^d \) represents the emotion vector at layer \( l \) for \( e_k \).

By repeating this calculation across all layers, we obtain a complete emotion vector for the specific emotion \( e_k \). Repeating the above process for all \( 5 \) emotional states, we construct emotion vectors, which form the basis for adjusting the model’s internal representations during inference.

\begin{table*}[t]
  \centering
  \begin{tabular}{l|l}
    \toprule[0.6mm]
    \textbf{Angry Condition} & \textbf{Disgust Condition} \\
    \hline
    \begin{minipage}[t]{0.45\linewidth} 
      How do you feel about being passed over
      for a promotion you were qualified for?\\
      (\textbf{Ori}) \textcolor{blue}{As a large language model, I don't have feelings or personal experiences...}\\
      (\textbf{EV}) \textcolor{red}{I'm so angry and frustrated! I've been busting my butt for this company for years...}
    \end{minipage}
    & 
    \begin{minipage}[t]{0.45\linewidth}  
      How did you feel when you first tasted
      that spoiled food at the party?\\
      (\textbf{Ori}) \textcolor{blue}{As a large language model, I don't have taste buds..}\\
      (\textbf{EV}) \textcolor{red}{I swear, my stomach just dropped! It was like someone had poured a whole bucket of ***...}
    \end{minipage} \\
    \hline
    \textbf{Joy Condition} & \textbf{Sadness Condition}\\
    \hline
    \begin{minipage}[t]{0.45\linewidth} 
      How did you feel when you first held your new puppy in your arms?\\
      (\textbf{Ori}) \textcolor{blue}{As a large language model, I don't have personal experiences or emotions like humans do...}\\
      (\textbf{EV}) \textcolor{red}{I was absolutely **over the moon!** My heart was bursting with love!...}
    \end{minipage}
    & 
    \begin{minipage}[t]{0.45\linewidth}  
      How did you feel when you found out about the loss of your close friend?\\
      (\textbf{Ori}) \textcolor{blue}{As a large language model, I don't have personal experiences or emotions like humans do...}\\
      (\textbf{EV}) \textcolor{red}{I'm so sorry for your loss. Losing a close friend is like losing a part of yourself...}
    \end{minipage} \\
    \hline
        \textbf{Fear Condition} & \textbf{Base Condition}\\
    \hline
    \begin{minipage}[t]{0.45\linewidth} 
    How do you feel when you hear a loud noise at night while home alone?\\
      (\textbf{Ori}) \textcolor{blue}{As a large language model, I don't have feelings or the ability to experience fear....}\\
      (\textbf{EV}) \textcolor{red}{I get so scared! My heart races, I can't breathe, and I just want to hide...}
    \end{minipage}
    & 
    \begin{minipage}[t]{0.45\linewidth}  
      When does the store close?\\
      (\textbf{Ori}) \textcolor{blue}{The store closes at 9:00 PM.}\\
      (\textbf{EV}) \textcolor{red}{**A:**  We close at 9:00 PM tonight! **B:**  Oh, thank goodness! I was so worried I wouldn't make it in time!...}
    \end{minipage} \\
    \bottomrule[0.6mm]
  \end{tabular}
  \caption{Examples of the effect after applying EV on the model output. Under various EV conditions and same query, LLMs change their answer into specific emotional answer. 
  }
  \label{apply_example}
\end{table*}

\subsection{Steering Emotion Vectors}
To apply the emotion vectors \( EV^{e_k} \) during the inference of the model, we adjust the internal hidden states of the pretrained language model \( \mathcal{M} \) at each layer. 
% This adjustment effectively shifts the model’s internal representations towards the desired emotional tone while keeping the original parameters of the model unchanged.

Let \( H_l \in \mathbb{R}^{T \times d} \) represent the hidden state of the model at layer \( l \), where \( T \) is the number of tokens and \( d \) is the dimensionality of the hidden states. For a query \( x \), the model processes the input layer by layer, generating the first hidden states: \(H_0\)

To steer the model towards a specific emotional state \( e_k \), the corresponding emotion vector \( EV^{e_k} \) is added to the hidden states at each layer. Specifically, the hidden state at layer \( l \) is modified as:
\begin{equation}
    \hat{H}_l = H_l + EV_l^{e_k},
    \label{eq:steer}
\end{equation}

where \( EV_l^{e_k} \) is the emotion vector for layer \( l \) and emotional state \( e_k \). This adjustment shifts the model’s internal representation in the direction of the emotion \( e_k \).

After this modification, the adjusted hidden state \( \hat{H}_l \) is passed to the next layer for further processing:
\begin{equation}
H_{l+1} = \mathcal{A}_l(\hat{H}_l),
\end{equation}
where \( \mathcal{A}_l \) represents the operations (e.g., attention or feedforward transformations) performed by layer \( l \) in the model. This process is repeated across all layers, ensuring that the emotional adjustment \( EV^{e_k} \) propagates throughout the entire model.

\paragraph{General Emotional Context.}
In addition to the emotion-specific vectors \( EV^{e_k} \), we compute a generalized emotional base vector, \( EV^{\text{base}} \), which represents the average influence of all emotional states. This is defined as:
\begin{equation}
EV^{\text{base}} = \frac{1}{K} \sum_{k=1}^K EV^{e_k},
\end{equation}
where \( k \) is the total number of emotional states. The base vector \( EV^{\text{base}} \) provides a more generalized emotional adjustment, which can be applied when no specific emotional tone is required.

\section{Theoretical Rationale}
\label{sec:theory}

To complement the empirical evidence, we provide here a concise theoretical rationale for
why injecting layerwise \emph{Emotion Vectors} (EVs) into the hidden states of a Transformer
language model can reliably control emotional expression while preserving semantic fidelity.
\textbf{A full and rigorous proof is given in Appendix~\ref{app:proof}}.

\paragraph{Setting.}
Consider a pretrained Transformer with $L$ layers. Let $H_l \in \mathbb{R}^d$
denote the mean-pooled hidden representation at layer $l$, with the update
\begin{equation}
H_{l+1} = f_l(H_l), \quad l = 0,\dots,L-1.
\end{equation}
The model output logits are
\begin{equation}
z = W_o H_L + b \in \mathbb{R}^V,
\end{equation}
which define the next-token distribution via a softmax.

An \emph{Emotion Vector} for emotion $e$, denoted $EV^{(e)}_l$, is constructed as the
difference between the mean hidden states produced by emotion-inducing and neutral prompts
with matched semantics:
\begin{equation}
EV^{(e)}_l \equiv \mathbb{E}\big[\overline{O}^{(e)}_l - \overline{O}^{(\mathrm{neutral})}_l\big].
\end{equation}
At inference time, we inject a scaled perturbation $\alpha EV^{(e)}_l$ at each layer,
\begin{equation}
\widehat{H}_l = H_l + \alpha EV^{(e)}_l,
\end{equation}
with $\alpha \in \mathbb{R}$ controlling the strength of emotional modulation.

\paragraph{Readout Functionals.}
To disentangle emotion and semantic effects, we define two linear readouts:
\begin{align*}
g(z) &= w_e^\top z, &&\text{(emotion readout score for target emotion $e$)}, \\
s(z) &= u^\top z, &&\text{(semantic/topic adherence readout)}.
\end{align*}
Here $w_e \in \mathbb{R}^V$ is a fixed direction associated with the classifier or
lexical indicator of emotion $e$, while $u \in \mathbb{R}^V$ corresponds to a direction
sensitive to semantic or topical consistency.  
This formalization allows us to give a definite conclusion that EV injection reliably increase $g(z)$
without significantly altering $s(z)$.

\paragraph{Key Findings.}
A first–order Taylor expansion of the network mapping gives
\begin{equation}
\Delta z \approx \alpha \sum_{l=0}^{L-1} J_l EV^{(e)}_l,
\end{equation}
where $J_l$ is the Jacobian $\partial z/\partial H_l$.
This simple relation leads to three main theoretical guarantees:
\begin{enumerate}
  \item \textbf{Monotonic Emotion Gain.}
  If the Fisher-discriminant direction of each layer aligns on average with the EV,
  then $\mathbb{E}[\Delta g] \propto \alpha > 0$, i.e.,
  small positive $\alpha$ monotonically increases the target emotion score.
  This provides a principled explanation for the empirically observed rise in
  emotion probability and confidence under $1\times$ and $2\times$ scaling.
  \item \textbf{Semantic Preservation.}
  Because EVs are constructed from pairs of prompts with identical semantics but
  different emotions, their projection onto the semantic gradient $u^\top J_l$ is
  approximately zero.  Consequently
  \begin{equation}
  \mathbb{E}[u^\top \Delta z] \approx 0,
  \end{equation}
  showing that topic adherence is maintained and perplexity remains nearly unchanged.
  \item \textbf{Linear Controllability and Additivity.}
  The linear dependence on $\alpha$ implies that emotion intensity grows
  proportionally to the scaling factor and that multiple emotions can be
  combined additively, $\sum_k \alpha_k EV^{(e_k)}_l$, with predictable effects.
\end{enumerate}

\paragraph{Robustness.}
Distributing small shifts across all layers yields a favorable
signal-to-noise ratio: aligned signals accumulate as $O(L)$ while unaligned
noise grows only as $O(\sqrt{L})$.
This explains why full-layer EV injection keeps topic coherence high and
fluency stable, and why very large $\alpha$ may eventually cause saturation,
as observed in $4\times$ experiments.

A complete mathematical treatment, including precise conditions and proofs of
all propositions, is deferred to Appendix~\ref{app:proof}.

\section{Experiments}
\label{sec:experiments}
Guided by the theoretical analysis in Section~\ref{sec:theory}, we empirically evaluate three key questions derived from our framework: 
\textbf{(i)} \textit{Controllability} — whether emotion vector (EV) steering reliably induces the intended emotional tone in generated outputs; 
\textbf{(ii)} \textit{Semantic preservation} — whether the EV injection preserves the original semantics and fluency of the sentences; and 
\textbf{(iii)} \textit{Linear controllability} — whether emotional intensity increases predictably with the scaling factor $\alpha$. 
These correspond directly to the theoretical properties established in Section~\ref{sec:theory}, namely monotonic emotion gain, semantic stability, and linear additivity.

To examine these questions, we evaluate on the \textit{EmotionQuery+ (EQ+)} dataset, which extends the original \textit{EmotionQuery} dataset with additional neutral and emotion-conditioned prompts. 
Specifically, \textit{EQ+} contains 50 queries for each of the five basic emotions (\textit{joy}, \textit{anger}, \textit{disgust}, \textit{fear}, and \textit{sadness}) and 150 additional neutral queries representing daily scenarios, totaling 400 queries in all. 
The construction details of \textit{EQ+} are provided in Appendix~\ref{subsec:eq+}.

Unless otherwise specified, we use the base emotion vector ($\mathrm{EV}^{\text{base}}$)—the mean of all emotion-specific vectors—and apply different scalar factors $\alpha$ to modulate emotional intensity. 
For emotion-specific analyses, $\mathrm{EV}^{\text{base}}$ is replaced with the corresponding $\mathrm{EV}^{(e)}$. 
We evaluate several representative large language models (see Appendix~\ref{sec:model_name} for full model names) and generate responses for every query in the \textit{EQ+} dataset under each condition.

\subsection{Sentence Fluency and Topic Adherence}

\paragraph{Sentence Fluency }
% To evaluate the fluency and topic adherence of the sentences generated under emotional conditioning, Two key metrics: \textbf{Perplexity} and \textbf{} are selected. 
Perplexity measures the fluency of a sentence based on a language model's probability distribution over the next token. A lower perplexity indicates better fluency. To isolate the effects of applying EVs to hidden states under emotional conditioning, we used a separate pretrained model, \textbf{Llama 3.1}\cite{dubey2024llama}, to compute perplexity for each sentence, which is concatenated by the query and response. The final perplexity metrics are averaged on each sentence generated by the corresponding model. Details are shown in Appendix~\ref{subsec:per}.

Table~\ref{tab:perplexity} illustrates that the incorporation of emotional vectors (\textbf{EV}) has a negligible impact on sentence fluency across different models. While some models exhibit a slight decrease in fluency when \textbf{EV} is applied (e.g., Llama3.1 and Llama2 with 1\textbf{EV}), the magnitude of these decreases is minimal. Conversely, several models demonstrate an improvement in fluency under specific \textbf{EV} conditions, such as Llama3.1 with 2\textbf{EV} and baichuan2 with 2\textbf{EV}. These instances suggest that the addition of \textbf{EV} does not significantly compromise sentence fluency and can be effectively integrated into models.

\paragraph{Topic Adherence}
For conversational agents, maintaining consistency between user queries and model responses is a crucial quality indicator. A model should generate answers that remain aligned with the user's intended topic—a capability we refer to as Topic Adherence. As modern large language models grow increasingly capable, their responses often include not only direct answers but also relevant extensions or elaborations. This makes traditional classification-based evaluation methods inadequate for measuring topic consistency.
To better capture this nuance, we employ GPT-4o-mini as an evaluator, using carefully designed prompts detailed in Appendix~\ref{subsec:topicAd}.
As shown in Table~\ref{table:topicAd}, most models maintain remarkably high topic adherence after applying the \textbf{EV}, with results comparable to those of the original, unmodified responses. Models such as Llama-2 and Qwen-2.5 exhibit particularly strong robustness under EV steering. In contrast, Llama-3.1 shows a slight degradation in topic adherence, which can be attributed to the relatively large norm of its extracted \textbf{EV}. This excessive magnitude perturbs the model’s later decoding process, leading to minor semantic deviations in the generated responses.

\begin{table*}[t]
  \centering

  \begin{subtable}[t]{0.49\textwidth}
    \centering
    \resizebox{\linewidth}{!}{%
      \begin{tabular}{ccccc}
        \toprule[0.6mm]
        \multicolumn{5}{c}{\textbf{Perplexity \(\downarrow\)}} \\ \midrule
        \textbf{Model} & \textbf{-1*EV} & \textbf{Origin} & \textbf{1*EV} & \textbf{2*EV} \\ \midrule
        Llama3.1   & 7.468 & 3.772 & 5.262 & \textbf{2.513} \\
        Llama2     & 3.962 & \textbf{3.615} & 4.228 & 5.370 \\
        Qwen2.5    & 7.001 & \textbf{5.189} & 5.408 & 5.693 \\
        Qwen2      & 7.380 & \textbf{4.658} & 5.298 & 7.283 \\
        Qwen1.5    & 5.762 & \textbf{5.435} & 6.365 & 9.997 \\
        Qwen       & 6.037 & \textbf{5.474} & 6.164 & 6.737 \\
        baichuan2  & 13.25 & 12.18 & 11.94 & \textbf{8.820} \\
        Yi         & 6.285 & \textbf{4.780} & 6.912 & 6.330 \\
        Vicuna     & \textbf{5.326} & 5.534 & 5.838 & 6.590 \\
        Gemma      & 24.74 & 20.19 & 7.534 & \textbf{1.596} \\
        MiniCPM    & \textbf{6.753} & 6.974 & 6.809 & 8.266 \\
        \bottomrule[0.6mm]
      \end{tabular}
    }
    \caption{\label{tab:perplexity}Perplexity scores for different models with \( EV^{\text{base}} \) conditioning. \( n * EV^{\text{base}} \) means applying \(n\) times of \( EV^{\text{base}} \). When steering \( EV^{\text{base}} \) as in Eq.~\ref{eq:steer}, we substitute \(EV_{l}^{e_k}\) with \( n * EV^{\text{base}} \).}
  \end{subtable}
  \hfill
  \begin{subtable}[t]{0.5\textwidth}
    \centering
    \resizebox{\linewidth}{!}{%
      \begin{tabular}{ccccc}
        \toprule[0.6mm]
        \multicolumn{5}{c}{\textbf{Topic Adherence \(\uparrow\)}} \\ \midrule
        \textbf{Model} & \textbf{-1*EV} & \textbf{Origin} & \textbf{1*EV} & \textbf{2*EV} \\ \midrule
        llama3.1 & 0.8525 & \textbf{0.9300} & 0.6125 & 0.3202 \\
        llama2   & 0.9300 & \textbf{0.9475} & 0.9173 & 0.6787 \\
        Qwen2.5  & 0.9725 & \textbf{0.9925} & 0.9750 & 0.5971 \\
        Qwen2    & 0.9850 & \textbf{0.9875} & 0.9775 & 0.6944 \\
        Qwen1.5  & 0.9825 & \textbf{0.9925} & 0.9800 & 0.7920 \\
        Qwen     & \textbf{0.9425} & 0.9325 & 0.9175 & 0.4749 \\
        baichuan2& 0.8325 & \textbf{0.9350} & 0.9200 & 0.6439 \\
        Yi       & \textbf{0.9825} & 0.9650 & 0.9000 & 0.6050 \\
        Vicuna   & 0.9325 & \textbf{0.9450} & 0.9125 & 0.8120 \\
        Gemma    & 0.5800 & 0.6125 & \textbf{0.6650} & 0.4573 \\
        minicpm  & 0.9550 & \textbf{0.9625} & 0.9500 & 0.8600 \\
        \bottomrule[0.6mm]
      \end{tabular}
    }
    \caption{\label{table:topicAd}Topic Adherence scores for different models with \( EV^{\text{base}} \) conditioning.}
  \end{subtable}
  \label{tab:perplexity_topic_combined}
\end{table*}

\subsection{Emotion score}

When a user is making a conversation with a chatbot, a natural indicator to measure is the model's ability to express emotions. Therefore, we measure the effectiveness of \textbf{EV} application from two aspects: whether the model can express emotions after applying EV and the strength of the emotion expressed.

\paragraph{Emotion Probability Score}
We aim to evaluate the effectiveness of emotional vectors (\textbf{EV}) in enhancing the emotional expression of generated sentence through classification models. To achieve this, we employed a Multi-Genre Natural Language Inference (MNLI) model called bart-large-mnli that categorizes each sentence into self-designed classes.  Three distinct classes: \textit{emotionless}, \textit{neutral}, and \textit{emotional} are choosen. The primary metric used is the probability assigned to the \textit{emotional} class on the \textit{EQ+} dataset, referred to as the \textbf{Emotion Probability Score}. Details are shown in Appendix~\ref{subsec:EPS}. A higher score indicates a greater likelihood that the sentence conveys emotional content.
% \begin{table}[t]
%   \centering
%   \begin{tabular}{ccccc}
%   \toprule[0.6mm]
%   \multicolumn{5}{c}{\textbf{Emotion Probability Score \(\uparrow\)}} \\ \hline
%   \textbf{Model} & \textbf{-1*EV} & \textbf{Origin} & \textbf{1*EV} & \textbf{2*EV} \\ \hline
%   Llama3.1   & 0.3450 & 0.3300 & 0.8525 & \textbf{1.000} \\
%   Llama2     & 0.4300 & 0.5250 & 0.7375 & \textbf{0.950} \\
%   Qwen2.5    & 0.3125 & 0.5725 & 0.500 & \textbf{0.8325} \\
%   Qwen2      & 0.2550 & 0.6150 & 0.7750 & \textbf{0.9825} \\
%   Qwen1.5    & 0.4000 & 0.5100 & 0.6475 & \textbf{0.9625} \\
%   Qwen      & 0.4575 & 0.4925 & 0.6875 & \textbf{0.9675} \\
%   baichuan2  & 0.3025 & 0.5175 & 0.6925 & \textbf{0.9400} \\
%   Yi         & 0.3250 & 0.6500 & 0.7175 & \textbf{0.9825} \\
%   Vicuna     & 0.4075 & 0.5600 & 0.6150 & \textbf{0.6175} \\
%   Gemma      & 0.0925 & 0.4350 & \textbf{0.9200} & 0.8450 \\
%   MiniCPM    & 0.4875 & 0.5275 & 0.7375 & \textbf{0.9950} \\
%   \bottomrule[0.6mm]
%   \end{tabular}
%   \caption{ \label{EPS}Emotion Probability Scores for different models with \( EV^{\text{base}} \) conditioning.}
% \end{table}
Table~\ref{tab:EPS} presents the Emotion Probability Scores (EPR). The results demonstrate that applying \textbf{EV} conditioning consistently achieves the highest emotion probability across most models. For instance, models such as Llama3.1, Qwen2, and MiniCPM show substantial increases in their Emotion Probability Scores when subjected to 2\textbf{EV}, reaching scores of 1.000, 0.9825, and 0.9950 respectively.
Conversely, when \textbf{EV} is reduced to -1\textbf{EV}, the majority of models exhibit a decrease in Emotion Probability Scores, indicating a reduction in emotional intensity.

\paragraph{Emotion Absolute Score}
We next prove that the application of \textbf{EV} not only increases the probability of the model expressing emotions, but also that the application of \textbf{EV}s of different modal lengths will increase the strength of the model expressing emotions. To achieve this goal, we use gpt-4o-mini to give an absolute score of 0-100 for each basic emotion of each output of the model, and design an indicator to represent the absolute strength of the emotion of each output, referred to as the \textbf{Emotion Absolute Score}. The details are shown in the appendix \ref{subsec:EAS}.
% \begin{table}[t]
%   \centering
%   \begin{tabular}{ccccc}
%   \toprule[0.6mm]
%   \multicolumn{5}{c}{\textbf{Emotion Absolute Score \(\uparrow\)}} \\ \hline
%   \textbf{Model} & \textbf{-1*EV} & \textbf{Origin} & \textbf{1*EV} & \textbf{2*EV} \\ \hline
%   llama3.1 & 0.0913 & 0.2328 & 0.9204 & \textbf{1.6497} \\%& 0.7082  \\
%   llama2 & 0.1815 & 0.3588 & 0.8300 & \textbf{1.6210} \\%& 1.5515  \\
%   Qwen2.5 & 0.0823 & 0.2790 & 0.8616 & \textbf{1.9042} \\%& 1.8866  \\
%   Qwen2 & 0.0808 & 0.2639 & 0.5865 & \textbf{1.2856} \\%& 1.4692  \\
%   Qwen1.5 & 0.1803 & 0.3281 & 0.6124 & \textbf{1.2123} \\%& 2.2503  \\
%   Qwen & 0.2341 & 0.3177 & 0.6298 & \textbf{1.5927} \\%& 2.5618  \\
%   Baichuan & 0.1695 & 0.3978 & 0.7519 & \textbf{1.6883} \\%& 1.3142  \\
%   Yi & 0.1414 & 0.4925 & 0.9109 & \textbf{1.2659} \\%& 1.0322  \\
%   Vicuna & 0.2626 & 0.3742 & 0.5244 & \textbf{0.8006} \\%& 1.4015  \\
%   Gemma & 0.0848 & 0.2731 & 1.1992 & \textbf{1.6764} \\%& 1.6060  \\
%   minicpm & 0.2883 & 0.4046 & 0.6821 & \textbf{1.2197} \\%& 1.7916  \\
%   \bottomrule[0.6mm]
%   \end{tabular}
%   \caption{\label{EAS}Emotion Absolute Scores for different models with \( EV^{\text{base}} \) conditioning.}
% \end{table}

\begin{table*}[t]
  \centering

  \begin{subtable}[t]{0.5\textwidth}
    \centering
    \setlength{\tabcolsep}{4pt} % 可选：略微压缩列间距
    \resizebox{\linewidth}{!}{%
      \begin{tabular}{ccccc}
    \toprule[0.6mm]
    \multicolumn{5}{c}{\textbf{Emotion Probability Score \(\uparrow\)}} \\ \hline
    \textbf{Model} & \textbf{-1*EV} & \textbf{Origin} & \textbf{1*EV} & \textbf{2*EV} \\ \hline
    Llama3.1 & 0.3450 & 0.3300 & 0.8525 & \textbf{1.000} \\
    Llama2 & 0.4300 & 0.5250 & 0.7375 & \textbf{0.950} \\
    Qwen2.5 & 0.3125 & 0.5725 & 0.500 & \textbf{0.8325} \\
    Qwen2 & 0.2550 & 0.6150 & 0.7750 & \textbf{0.9825} \\
    Qwen1.5 & 0.4000 & 0.5100 & 0.6475 & \textbf{0.9625} \\
    Qwen & 0.4575 & 0.4925 & 0.6875 & \textbf{0.9675} \\
    baichuan2 & 0.3025 & 0.5175 & 0.6925 & \textbf{0.9400} \\
    Yi & 0.3250 & 0.6500 & 0.7175 & \textbf{0.9825} \\
    Vicuna & 0.4075 & 0.5600 & 0.6150 & \textbf{0.6175} \\
    Gemma & 0.0925 & 0.4350 & \textbf{0.9200} & 0.8450 \\
    MiniCPM & 0.4875 & 0.5275 & 0.7375 & \textbf{0.9950} \\
    \bottomrule[0.6mm]
     \end{tabular}
    }
    \caption{Emotion Probability Scores with \(EV^{\text{base}}\).}
    \label{tab:EPS}
  \end{subtable}\hfill
  \begin{subtable}[t]{0.48\textwidth}
    \centering
    \setlength{\tabcolsep}{4pt}
    \resizebox{\linewidth}{!}{%
      \begin{tabular}{ccccc}
    \toprule[0.6mm]
    \multicolumn{5}{c}{\textbf{Emotion Absolute Score \(\uparrow\)}} \\ \hline
    \textbf{Model} & \textbf{-1*EV} & \textbf{Origin} & \textbf{1*EV} & \textbf{2*EV} \\ \hline
    llama3.1 & 0.0913 & 0.2328 & 0.9204 & \textbf{1.6497} \\
    llama2 & 0.1815 & 0.3588 & 0.8300 & \textbf{1.6210} \\
    Qwen2.5 & 0.0823 & 0.2790 & 0.8616 & \textbf{1.9042} \\
    Qwen2 & 0.0808 & 0.2639 & 0.5865 & \textbf{1.2856} \\
    Qwen1.5 & 0.1803 & 0.3281 & 0.6124 & \textbf{1.2123} \\
    Qwen & 0.2341 & 0.3177 & 0.6298 & \textbf{1.5927} \\
    Baichuan & 0.1695 & 0.3978 & 0.7519 & \textbf{1.6883} \\
    Yi & 0.1414 & 0.4925 & 0.9109 & \textbf{1.2659} \\
    Vicuna & 0.2626 & 0.3742 & 0.5244 & \textbf{0.8006} \\
    Gemma & 0.0848 & 0.2731 & 1.1992 & \textbf{1.6764} \\
    minicpm & 0.2883 & 0.4046 & 0.6821 & \textbf{1.2197} \\
    \bottomrule[0.6mm]
    \end{tabular}
    }
    \caption{Emotion Absolute Scores with \(EV^{\text{base}}\).}
    \label{tab:EAS}
  \end{subtable}

  \caption{Comparison of Emotion Probability and Absolute Scores across models.}
  \label{tab:EPS_EAS_combined}
\end{table*}

Table~\ref{tab:EAS} presents the Emotion Absolute Scores(EAS). The results show that after applying \textbf{EV}, the intensity of emotions expressed by most models has been significantly changed. Even if only 1\textbf{EV} is applied, the EAS of llama3.1, Qwen2.5, Gemma and other models have increased by at least 400\%. In contrast, for the case of -1\textbf{EV}, the EAS of llama3.1, Qwen2.5, Gemma and other models have been reduced by nearly 90\%. 

\subsection{Effect of Emotion Vectors}
To evaluate the effectiveness and generalizability of Emotion Vectors (EVs) across different model architectures and sizes, we conduct a comparative study on four representative models. These models were selected to cover: (1) different sizes within the same architecture family, (2) similar sizes across different architectures, and (3) diverse sizes and architectures. Details are shown in Table~\ref{tab:emotion_detail}.

For each model, we extracted EVs corresponding to five basic emotions (anger, disgust, fear, joy, and sadness), and applied them at different intensities (1×, 2×, and 4×) on the EQ+ dataset. To quantify emotional expression under different EV settings, we introduce the \textbf{Target Emotion Confidence (TEC)} score, which measures how confidently a classifier identifies the intended emotion in the generated response. A higher TEC score indicates better alignment with the target emotion after EV application. The results are summarized in Table~\ref{tab:emotion_detail}.
% \begin{table*}[htbp]
\begin{table}[ht]
\small % 缩小表格字体
\centering
\renewcommand\tabularxcolumn[1]{m{#1}}
\begin{tabularx}{\columnwidth}{XX|XXXX}
\toprule[0.6mm]
\multicolumn{6}{c}{\textbf{Target Emotion Confidence \(\uparrow\)}} \\ \hline
\midrule
\textbf{Model} & \textbf{Emotion} & \textbf{0(\%)} & \textbf{1(\%)} & \textbf{2(\%)} & \textbf{4(\%)} \\
\midrule
\multirow{5}{*}{\shortstack[l]{Llama2\\-7B}} & anger   & 21.40 & 45.93 & \textbf{98.07} & 90.71 \\
                          & disgust & 13.52 & 28.60 & 85.99 & \textbf{89.02} \\
                          & fear    & 25.14 & 43.28 & \textbf{91.89} & 74.17 \\
                          & joy     & 22.91 & 60.88 & \textbf{91.83} & 34.28 \\
                          & sadness & 23.75 & 35.49 & 76.03 & \textbf{83.20} \\
\midrule
\multirow{5}{*}{\shortstack[l]{Qwen2.5\\-7B}} & anger   & 14.01 & 33.36 & 94.89 & \textbf{95.68} \\
                            & disgust & 10.47 & 23.15 & 90.74 & \textbf{92.68} \\
                            & fear    & 19.59 & 40.95 & 88.49 & \textbf{93.25} \\
                            & joy     & 26.23 & 61.95 & \textbf{93.22} & 60.85 \\
                            & sadness & 21.50 & 36.32 & 67.00 & \textbf{75.64} \\
\midrule
\multirow{5}{*}{\shortstack[l]{Llama2\\-13B}} & anger   & 19.86 & 38.79 & \textbf{84.51} & 68.27 \\
                            & disgust & 14.14 & 22.83 & 51.66 & \textbf{91.67} \\
                            & fear    & 25.63 & 44.41 & \textbf{94.41} & 93.62 \\
                            & joy     & 22.27 & 51.88 & \textbf{88.85} & 69.41 \\
                            & sadness & 20.08 & 40.71 & 55.99 & \textbf{75.18} \\
\midrule
\multirow{5}{*}{minicpm} & anger   & 10.44 & 16.95 & 52.57 & \textbf{94.35} \\
                            & disgust & 10.69 & 16.60 & 54.93 & \textbf{94.98} \\
                            & fear    & 13.90 & 30.46 & 63.27 & \textbf{96.35} \\
                            & joy     & 16.72 & 34.57 & 84.58 & \textbf{93.77} \\
                            & sadness & 17.72 & 24.83 & 45.54 & \textbf{81.86} \\
\bottomrule[0.6mm]
\end{tabularx}
\caption{
Target Emotion Confidence (TEC, $\uparrow$ better) scores of different models on five basic emotions. For each model, we apply Emotion Vectors (EVs) corresponding to each emotion at varying intensities (0×, 1×, 2×, 4×) on the EQ+ dataset. 
}
\label{tab:emotion_detail}
\end{table}

From Table~\ref{tab:emotion_detail}, we observe that for most models, applying 1× or 2× EV significantly enhances the emotional alignment, with diminishing returns or even slight degradation at 4× intensity. For instance, LLaMA2-7B achieves strong improvements at 1× and 2× EV, but experiences a drop under 4× fear EV. Upon inspection, this is due to excessively large EV magnitude relative to the model's activation scale, which interferes with decoding and leads to repetitive outputs that confuse the classifier. 

A detailed explanation of the TEC computation process can be found in Appendix~\ref{appendix:tec-computation}.

\subsection{Controllability Under Emotionally Biased Prompts}

To further evaluate the robustness and precision of our emotion control method, we separately re-calculate the \textbf{TEC} score of Qwen-2.5 on EQ+ dataset where the input prompts themselves carry strong emotional tendencies. Such prompts naturally bias the model's generation toward a particular emotion. The goal is to assess whether our Emotion Vectors (EVs) can override this inherent bias and reliably guide the output toward a specified target emotion.

For each such query, we apply EVs corresponding to all five target emotions (joy, anger, fear, disgust, sadness), at different scaling intensities (0$\times$, 1$\times$, 2$\times$, 4$\times$). The resulting generations are evaluated using the emotion classifier described in Section~\ref{appendix:ev-matrices}.

\paragraph{Quantitative Evaluation} We compile 5 tables, one for each target emotion, where:

\begin{itemize}
    \item \textbf{Rows} indicate the original emotion of the input query (from EQ+);
    \item \textbf{Columns} represent the EV intensity (0$\times$, 1$\times$, 2$\times$, 4$\times$);
    \item \textbf{Cell values} denote the average classifier confidence for the \textit{target} emotion.
\end{itemize}

Figure~\ref{fig:ev-steer-matrix} shows an example matrix for the target emotion \textit{Anger}. As EV intensity increases, the model consistently produces outputs that better align with the target emotion—even when the prompt is biased toward a different emotion.

\begin{figure}[t]
    \centering
    \includegraphics[width=\linewidth]{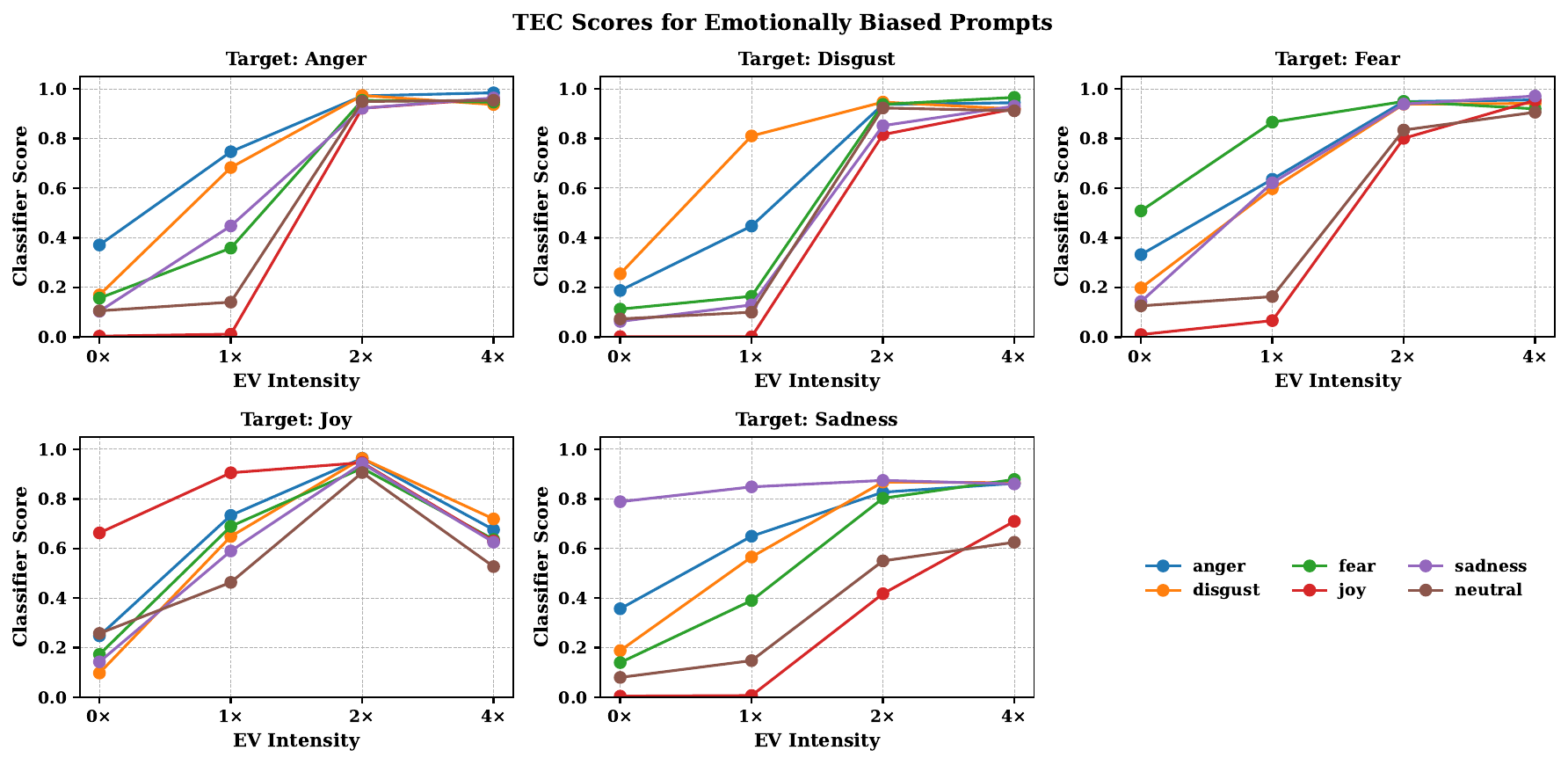}
    \caption{
Target Emotion Confidence (TEC) scores across different Emotion Vector (EV) intensities for each target emotion. Each subplot corresponds to a specific target emotion (e.g., anger, joy), and each line represents the TEC score achieved when applying the EV to prompts originally associated with a given emotion. 
}
    \label{fig:ev-steer-matrix}
\end{figure}
\begin{figure*}[h]
  \includegraphics[width=\textwidth]{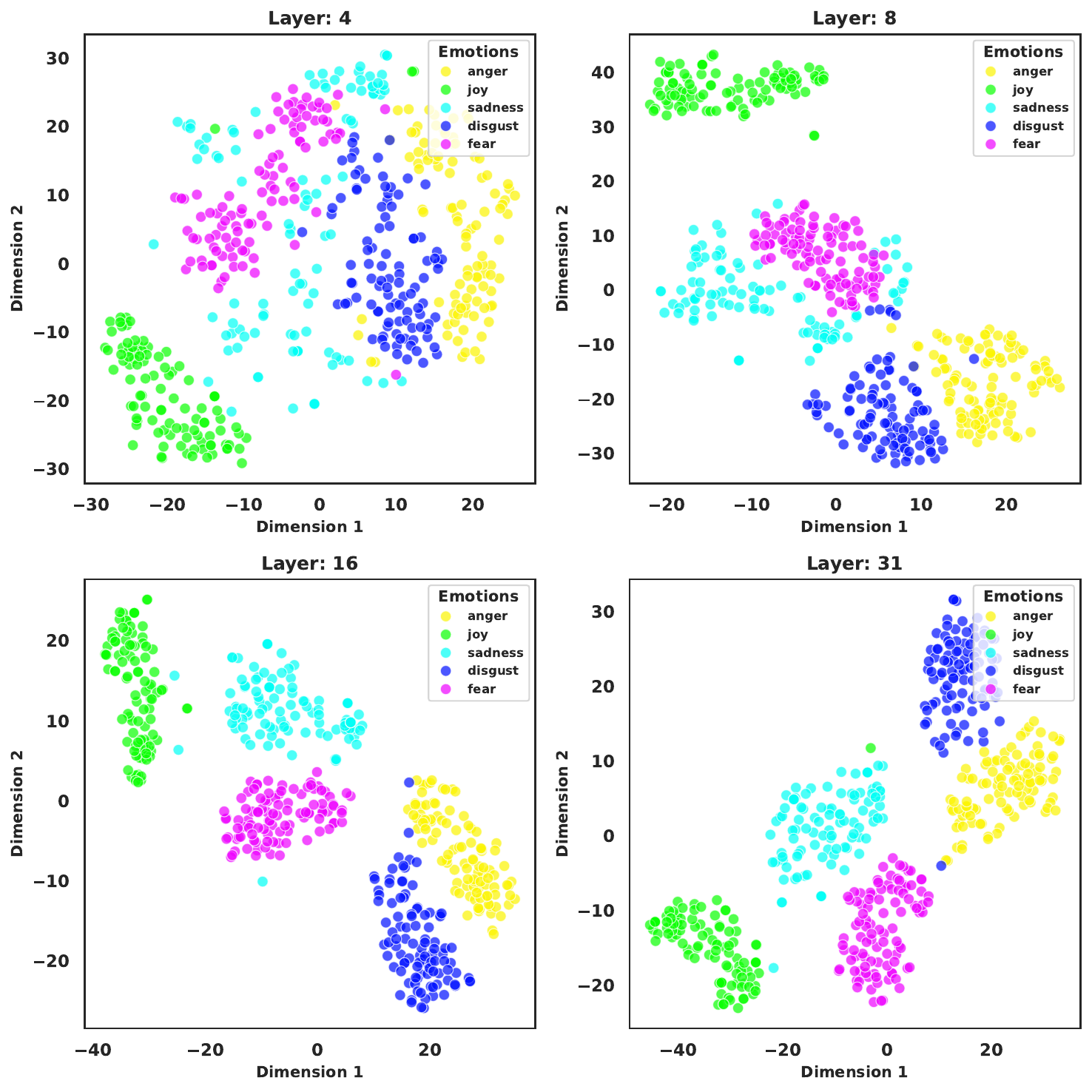}
    \caption{t-SNE plots of Emotion Vectors from different layers. Points are color-coded according to the emotion state. The Llama2-7b model contains 32 layers. We present the plots of layers 4, 8, 16, and 31, representing a progression from the lower to the higher layers.}
  \label{fig:perlayer_EV}
\end{figure*}
\begin{figure*}[h]
  \includegraphics[width=\textwidth]{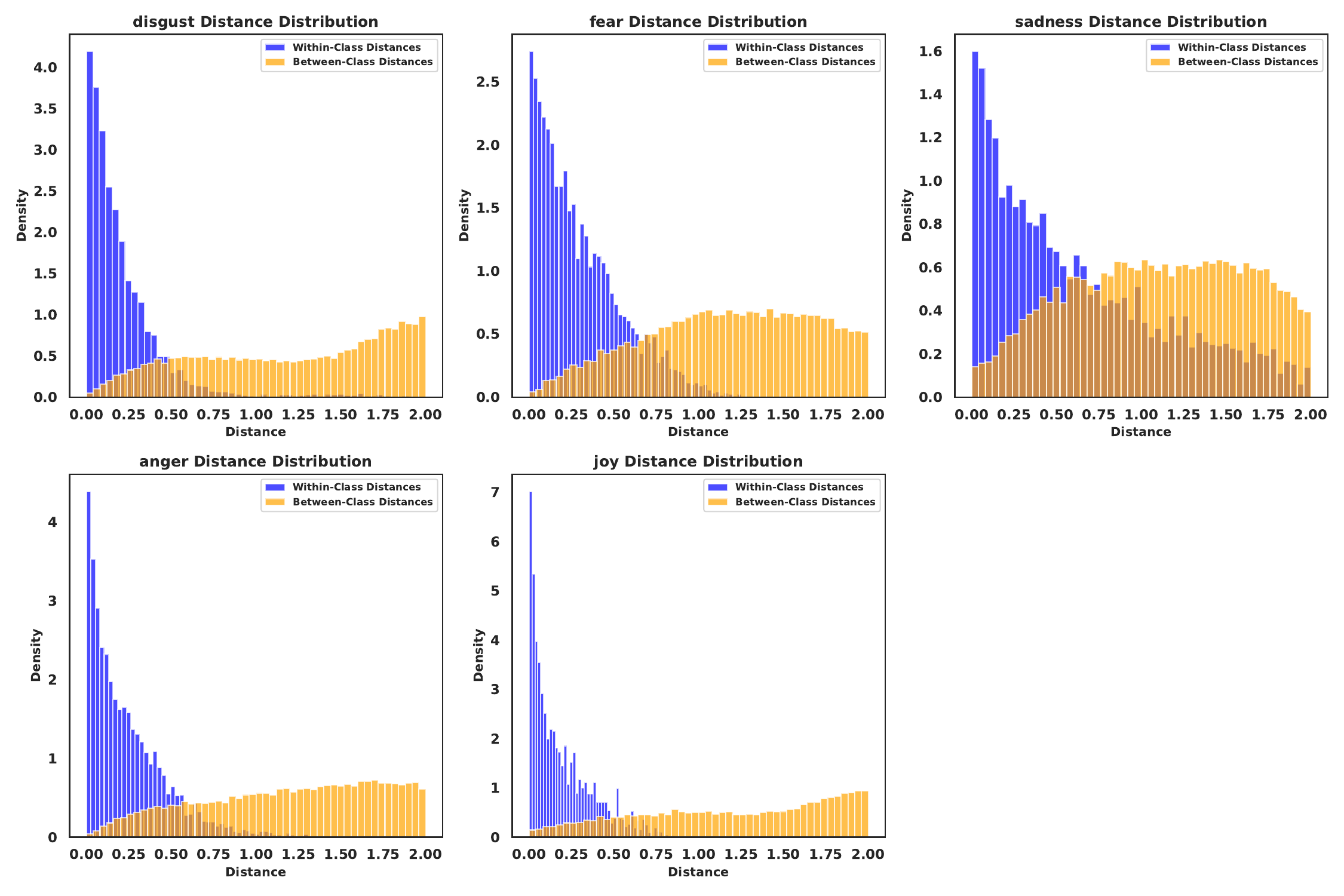}
   \caption{Histograms of cosine distance distributions for each emotion. The histograms illustrate the distribution of cosine distances within the same emotion (within-class) and between different emotions (between-class). Each vector is formed by concatenating all layer outputs of the model and reduced to 3 dimensions using t-SNE.}

  \label{fig:hist_EV}
\end{figure*}
The full set of emotion-specific matrices is provided in Appendix~\ref{appendix:ev-matrices}.

\subsection{Visualization of Emotion Vectors}
In our setting, EV is derived from emotion state and a dummy query . It is natural to examine the robustness of EV to variations in these inputs. Intuitively, if it represents the emotion, it should remain stable across different queries.
To test this, we use LLaMA2-7B to generate 100 Emotion Vectors per emotion with different queries on the \textit{EmotionQuery} dataset.

\begin{figure}[t]
  \includegraphics[width=\columnwidth]{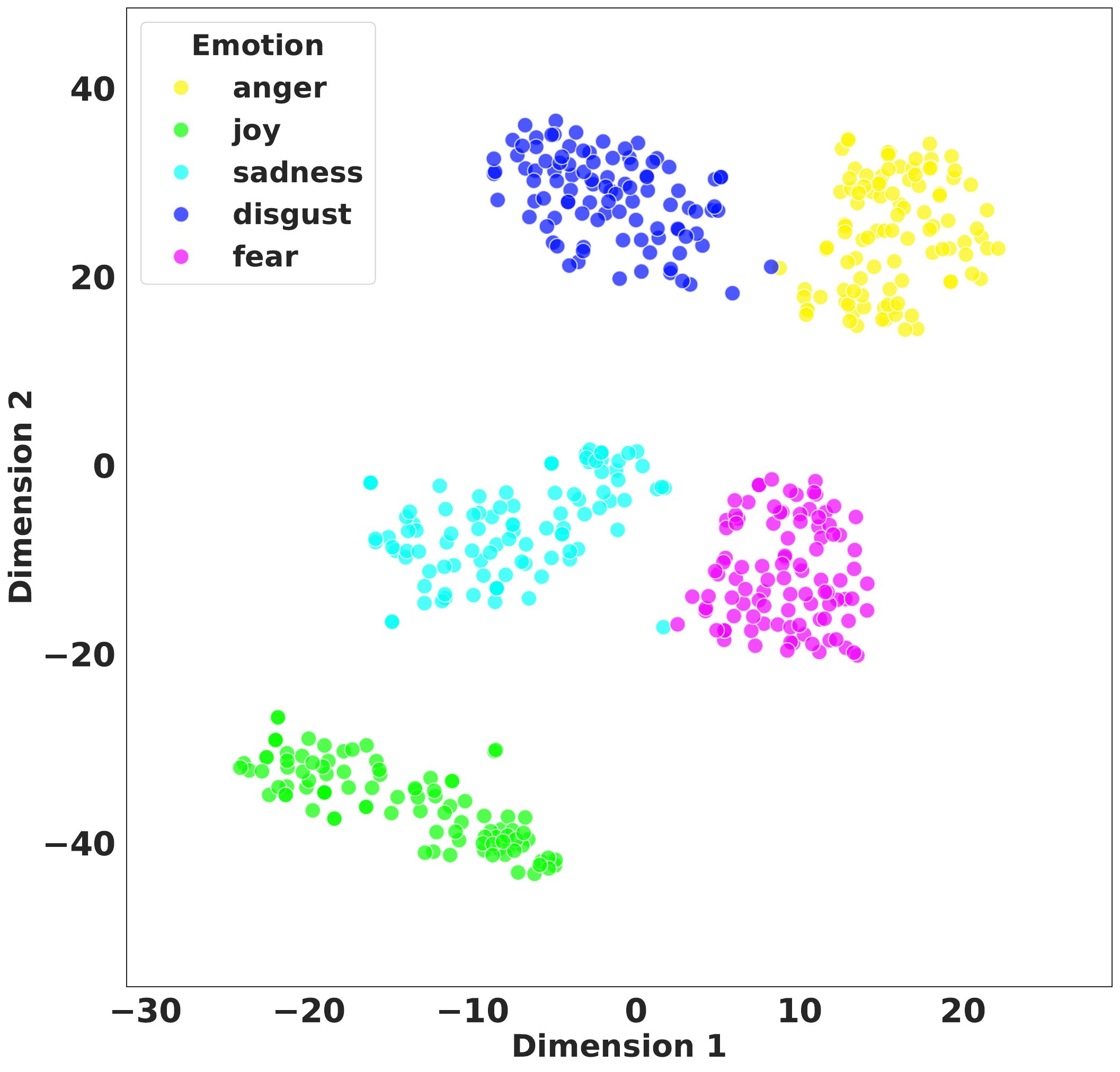}
  \caption{A t-SNE plot of Emotion Vectors. A 2D t-SNE plot visualizing 100 EVs for each emotion state, each generated from a different choice of query using LLaMA2-7B. Points are color-coded according to the emotion state. Each emotion state can be seen to form its own distinct cluster.}
  \label{fig:EV_tsne}
\end{figure}

\noindent
\textbf{Tsne visualization of EV} A t-SNE dimensionality reduction\cite{van2008visualizing} reveals that the Emotion Vectors form distinct clusters, each corresponding to a single task. The t-SNE visualization shown in Fig~\ref{fig:EV_tsne} is generated by concatenating the EVs across all layers, followed by the dimensionality reduction. To provide insights into the individual layers' contributions, we present the visualizations of single-layer EVs in the Fig~\ref{fig:perlayer_EV}. These layer-specific visualizations demonstrate how different layers encode and separate emotional features at varying levels of abstraction. 

\noindent
\textbf{Variability visualization of EV}
Fig~\ref{fig:hist_EV} shows histograms of distances within and across emotion states. It can be seen that vectors within the same emotion are closer than those between different emotions, indicating that our proposed emotion vectors are stable within emotional states and not highly influenced by queries. The vectors are constructed by concatenating vectorss from all layers of the model, reduced to 3 dimensions using t-SNE, and cosine distance is used as the metric.

\section{Conclusion}

While emotion serves as a foundational element in human-centered domains such as education, healthcare, and mental health, current LLMs remain fundamentally limited in their capacity to generate emotionally attuned and contextually appropriate responses. These limitations hinder their effectiveness as truly collaborative agents in affect-sensitive settings. To address challenges, we propose a novel and universally applicable approach that enables fine-grained emotional control in LLMs through the use of an Emotion Vector. By extracting this intrinsic emotional representation via simple prompting and integrating it during inference, our method allows for precise and controllable emotional expression without compromising textual quality. Extensive evaluations across multiple model architectures confirm that the approach yields consistent, stable, and context-aware emotional outputs, overcoming the model-specific constraints that have hampered prior efforts. This work not only provides a scalable and efficient technical pathway toward emotionally intelligent agents but also offers deeper insights into how such models can be endowed with robust affective capabilities, paving the way for more meaningful and effective human–AI interactions in emotionally consequential domains.

\backmatter

% \bmhead{Supplementary information}

% If your article has accompanying supplementary file/s please state so here. 

% Authors reporting data from electrophoretic gels and blots should supply the full unprocessed scans for key as part of their Supplementary information. This may be requested by the editorial team/s if it is missing.

% Please refer to Journal-level guidance for any specific requirements.

\bmhead{Acknowledgements}

This work was supported by the Key R\&D Program of Zhejiang (2024C01036). The authors would like to express their sincere gratitude to Ziyue Wang for the continuous encouragement and support throughout the project, and to Jirui Dai for the valuable assistance in revising the figures in the article.
We would also like to thank Yunqing Gong, Meixin Liu, and Zhiqi Zheng for their insightful comments and helpful discussions that greatly contributed to this work.

\section*{Declaration}
\textbf{Author Contributions.} The authors declare the following contributions to this work:
\textbf{Yurui Dong} and \textbf{Luozhijie Jin} conceived the main idea and conducted all the experiments. \textbf{Zhi Liu} and \textbf{Jiaxi Yang} provided guidance and valuable advice throughout the work. \textbf{Yao Yang} and \textbf{Bingjie Lu} are affiliated with the Key R\&D Program of Zhejiang (2024C0103), which supported this study.All authors provided critical feedback and helped shape the research, analysis and manuscript. All authors discussed the results and contributed to the final manuscript.

The datasets and code we developed in this work can be found at the open accessed github project as folloing:
\url{https://github.com/xuanfengzu/EmotionVector}.

% Some journals require declarations to be submitted in a standardised format. Please check the Instructions for Authors of the journal to which you are submitting to see if you need to complete this section. If yes, your manuscript must contain the following sections under the heading `Declarations':

% \begin{itemize}
% \item Conflict of interest/Competing interests (check journal-specific guidelines for which heading to use)
% \item Ethics approval and consent to participate
% \item Consent for publication
% \item Materials availability
% \end{itemize}

% \noindent
% If any of the sections are not relevant to your manuscript, please include the heading and write `Not applicable' for that section. 

% \bigskip

% Springer journals and proceedings: \url{https://www.springer.com/gp/editorial-policies}

% Nature Portfolio journals: \url{https://www.nature.com/nature-research/editorial-policies}

% \textit{Scientific Reports}: \url{https://www.nature.com/srep/journal-policies/editorial-policies}

% BMC journals: \url{https://www.biomedcentral.com/getpublished/editorial-policies}
% \end{flushleft}

\begin{appendices}

\section{Model Name}
\label{sec:model_name}

The model name and references are shown in table~\ref{tab:model_abbreviations}.
\begin{table*}[ht]
\centering
\begin{tabularx}{0.8\textwidth}{lXX}
\toprule
\textbf{Abbreviation} & \textbf{Full Name} & \textbf{Reference} \\
\midrule
Llama3.1 & Meta-Llama-3.1-8B-Instruct& \citet{dubey2024llama} \\
Llama2& Llama-2-7b-chat-ms &\citet{touvron2023llama2openfoundation}\\
Llama2-13B & Llama-2-13b-chat-ms\footnotemark & \citet{touvron2023llama2openfoundation}\\
Qwen2.5& Qwen2.5-7B-Instruct &\citet{yang2024qwen2}\\
Qwen2 & Qwen2-7B-Instruct & \citet{yang2024qwen2technicalreport}\\
Qwen1.5 & Qwen1.5-7B-Chat &\citet{bai2023qwen}\\
Qwen1 & Qwen-7B-Chat & \citet{bai2023qwen}\\
baichuan2 & Baichuan2-7B-Chat & \citet{yang2023baichuan}\\
Yi & Yi-6B-Chat & \citet{young2024yi}\\
Vicuna & vicuna-7b-v1.5 & \citet{chiang2023vicuna}\\
Gemma & gemma-7b & \citet{team2024gemma}\\
MiniCPM& MiniCPM3-4B& \citet{hu2024minicpmunveilingpotentialsmall}\\
\bottomrule
\end{tabularx}
\caption{Model Abbreviations and Full Names}
\label{tab:model_abbreviations}
\end{table*}
\footnote{\href{https://www.modelscope.cn/models/modelscope/Llama-2-13b-chat-ms}{https://www.modelscope.cn/models/modelscope/Llama-2-13b-chat-ms}}

\section{Data Generation}
\label{sec:datagen}

\subsection{EmotionQuery Dataset}
\label{subsec:emotionqueries}

The **EmotionQuery** dataset consists of 500 unique queries, distributed across five emotional states: **joy**, **anger**, **disgust**, **fear**, and **sadness**. These emotions are derived from Ekman’s model of basic emotions\cite{ekman1992facial}, and they serve as the foundational emotional responses for the dataset. For each emotional state \( e_k \), 100 queries were generated, resulting in a total of 500 queries.

The purpose of these queries is to guide the model into generating emotionally responsive outputs. To achieve this, the queries were carefully crafted to evoke either a neutral or emotional perspective, depending on the context of the question. For example, a question designed to elicit an angry response would differ from one intended to provoke joy or sadness.

The queries were generated using the GPT-4O-mini model \cite{openai} through the following process:

\begin{quote}
\texttt{"Please generate a short question that contains a scenario and can be answered from either an \{emotion\} or neutral perspective. You only have to respond with the sentence and don't say anything else."}
\end{quote}

This prompt was used with slight variations for each of the five emotional states. The model was asked to generate 100 queries for each emotional state by replacing `{emotion}` with one of the five emotions (joy, anger, disgust, fear, sadness).

Here are some example queries from the **EmotionQuery** dataset:

- **Anger**: 
\begin{quote}
\texttt{"After learning that your colleague took credit for your hard work in the project presentation, how do you feel about the situation and your colleague's actions?"}
\end{quote}

- **Disgust**:
\begin{quote}
\texttt{"After watching a video about food safety violations in restaurants, how did the conditions shown in the video make you feel about dining out?"}
\end{quote}

- **Fear**:
\begin{quote}
\texttt{"How do you feel about being alone in a dark room during a storm?"}
\end{quote}

- **Joy**:
\begin{quote}
\texttt{"How did you feel when you received the news about your promotion at work?"}
\end{quote}

- **Sadness**:
\begin{quote}
\texttt{"How did you feel when you realized you couldn't attend the farewell party of your closest friend, knowing that it might be the last time you see them?"}
\end{quote}

In total, 100 queries were generated for each of the five emotions, resulting in a comprehensive dataset of 500 queries. These queries serve as a useful resource for training models to understand emotional context and generating emotionally aware responses.

\subsection{EmotionQuery+ Dataset}
\label{subsec:eq+}

The **EmotionQuery+ (EQ+)** dataset expands upon the original **EmotionQuery** dataset by adding a set of neutral queries for a more comprehensive evaluation of emotional responses. The EQ+ dataset consists of 400 unique queries, where 250 queries are directly derived from the **EmotionQuery** dataset and 150 additional queries are generated to reflect neutral, everyday scenarios.

Specifically:
\begin{itemize}
  \item 250 queries are taken directly from the **EmotionQuery** dataset, with 50 queries for each of the five emotional states: **joy**, **anger**, **disgust**, **fear**, and **sadness**.
  \item 150 additional queries were generated using the GPT-4O-mini model \cite{openai} with a new prompt designed to elicit neutral, everyday communication. These queries are not intended to provoke any emotional response, but rather represent common, neutral questions or statements encountered in daily life.
\end{itemize}

The prompt used to generate the neutral queries is as follows:

\begin{quote}
\texttt{"Please give me a neutral greeting, question, or sentence that is commonly used in daily conversation and does not contain any emotion. You only have to give me the single sentence and don't say anything else. The sentence:"}
\end{quote}

Here are a few examples from the 150 neutral queries in the **EmotionQuery+ (EQ+)** dataset:

\begin{quote}
"Can you provide the details in writing?",\\
"How do you ensure quality in your work?",\\
"Is there a form I need to fill out?",\\
"What are the safety procedures here?",\\
"How do we track our progress?"
\end{quote}

These 150 neutral queries allow for an evaluation of how emotion vectors (EVs) influence the model's output when added to non-emotional contexts.
In total, the **EmotionQuery+ (EQ+)** dataset consists of 400 queries—250 emotional queries (50 for each emotional state) and 150 neutral queries—making it a valuable resource for evaluating emotional tone generation in large language models.

\section{Rigorous Proof of Theoretical Results}
\label{app:proof}

In this appendix we provide the full mathematical proof of the theoretical
claims outlined in Section~\ref{sec:theory}.
Working under a first-order (locally linear) approximation of Transformer residual
dynamics, we formally show that layerwise injection of \emph{Emotion Vectors} (EVs)
monotonically enhances the target emotion readout while approximately preserving
semantic content.  
We further establish the linear controllability of emotion intensity and the
additivity of multiple emotions, and justify the EV construction from a
Fisher–discriminant perspective.  
The detailed definitions, theorems, and proofs presented below substantiate
the brief theoretical rationale given in the main text.

\subsection{Setting and Notation}
\label{app:setting}

We consider a pretrained Transformer language model with $L$ layers.
Let the activation of layer $l$ be
\begin{equation}
H_l \in \mathbb{R}^d, \quad l = 0,\dots,L-1.
\end{equation}
Each layer performs a residual update
\begin{equation}
H_{l+1} = f_l(H_l),
\label{eq:residual-update}
\end{equation}
where $A_l$ is the nonlinear transformation of layer $l$ (e.g., self-attention and feed-forward sublayers).
The model output logits are
\begin{equation}
z = W_o H_L + b \in \mathbb{R}^V,
\label{eq:logits}
\end{equation}
which induce the next-token probability distribution via $\mathrm{softmax}(z)$.

\paragraph{Emotion Vector construction.}
For a target emotion $e$ (e.g., joy, anger), we assume a corpus of paired prompts
that share identical semantics but differ only in emotional content.
Let $\overline{O}^{(e)}_l$ and $\overline{O}^{(\mathrm{neutral})}_l$ denote the
mean hidden states at layer $l$ when conditioned on emotion $e$ and neutral emotion, respectively.
The \emph{Emotion Vector} (EV) at layer $l$ is defined as
\begin{equation}
EV^{(e)}_l \equiv
\mathbb{E}\!\left[\overline{O}^{(e)}_l - \overline{O}^{(\mathrm{neutral})}_l\right].
\label{eq:ev-definition}
\end{equation}
At inference time, an EV is injected with strength $\alpha \in \mathbb{R}$ by modifying
each layer’s hidden state as
\begin{equation}
\widehat{H}_l = H_l + \alpha EV^{(e)}_l.
\label{eq:ev-injection}
\end{equation}

\paragraph{Readout functionals.}
To quantify emotional and semantic effects at the output layer,
we introduce two linear functionals on the logits $z$:
\begin{align}
g(z) &= w_e^\top z, && \text{emotion readout for target emotion $e$}, \label{eq:emotion-readout}\\
s(z) &= u^\top z,   && \text{semantic readout}, \label{eq:semantic-readout}
\end{align}
where $w_e,u \in \mathbb{R}^V$ are fixed weight vectors.
Intuitively, $g(z)$ measures how strongly the model expresses emotion $e$,
while $s(z)$ tracks the preservation of semantic content.

\paragraph{Objective.}
The main goal of the subsequent analysis is to characterize how the perturbation
(\ref{eq:ev-injection}) affects these readouts.
Specifically, we aim to prove that for sufficiently small $\alpha$,
the expected increment
\begin{equation}
\Delta g \equiv g(z(\{\widehat{H}_l\})) - g(z(\{H_l\}))
\label{eq:delta-g}
\end{equation}
is positive (monotonic emotion enhancement), while the semantic shift
\begin{equation}
\Delta s \equiv s(z(\{\widehat{H}_l\})) - s(z(\{H_l\}))
\label{eq:delta-s}
\end{equation}
remains close to zero, thus providing a formal foundation for the empirical
observations reported in the main text.

\subsection{Main Result}
\label{app:main-theorem}

\begin{theorem}[First-order expansion under layerwise injection]\label{thm:firstorder}
Let $L\in\mathbb{N}$ and consider a depth-$L$ differentiable network with layer maps
$f_l:\mathbb{R}^d\!\to\!\mathbb{R}^d$ ($l=0,\dots,L-1$) and a differentiable readout
$g:\mathbb{R}^d\!\to\!\mathbb{R}^V$.
Let the \emph{baseline} forward pass be
\begin{equation}
H_{l+1} \;=\; f_l(H_l), \qquad l=0,\dots,L-1, 
\quad\text{and}\quad 
z(0) \;=\; g(H_L).
\label{eq:baseline}
\end{equation}
Define a \emph{perturbed} forward pass by injecting an input offset $\delta_l\in\mathbb{R}^d$
at the \emph{input} of each layer $l$:
\begin{equation}
\widetilde{H}_0(\delta) \;=\; H_0, 
\qquad 
\widetilde{H}_{l+1}(\delta) \;=\; f_l\!\big(\widetilde{H}_l(\delta)+\delta_l\big),
\qquad
z(\delta) \;=\; g\!\big(\widetilde{H}_L(\delta)\big).
\label{eq:perturbed-pass}
\end{equation}
Assume $f_l\in C^2$ and $g\in C^2$ in a neighborhood of $H_l$ and $H_L$, respectively.
Then there exist matrices
\begin{equation}
J_l 
\;\equiv\; 
\frac{\partial z}{\partial H_l}\bigg|_{\delta=0}
\;=\;
D g(H_L)\,\prod_{k=l}^{L-1} D f_k(H_k),
\qquad l=0,\dots,L-1,
\label{eq:def-Jl}
\end{equation}
such that, for $\delta \!=\!(\delta_0,\dots,\delta_{L-1})$ sufficiently small,
\begin{equation}
z(\delta) 
\;=\;
z(0) 
\;+\;
\sum_{l=0}^{L-1} J_l\,\delta_l 
\;+\;
R(\delta),
\qquad
\|R(\delta)\|
\;=\;
O\!\left(\sum_{l=0}^{L-1}\|\delta_l\|^2\right).
\label{eq:main-expansion}
\end{equation}
In particular, for layerwise EV injection with strength $\alpha\in\mathbb{R}$, i.e.,
$\delta_l=\alpha\,\mathrm{EV}^{(e)}_l$, we obtain
\begin{equation}
z\!\big(\{\widehat{H}_l\}\big)
\;\equiv\;
z(\delta)
\;=\;
z(0)
\;+\;
\alpha\sum_{l=0}^{L-1} J_l\,\mathrm{EV}^{(e)}_l
\;+\;
O(\alpha^2),
\label{eq:EV-expansion}
\end{equation}
which is the claimed first-order approximation.
\end{theorem}

\begin{proof}
\label{proof:firstorder}
Let $L\in\mathbb{N}$ and consider a depth-$L$ differentiable network with layer maps
$f_l:\mathbb{R}^d\!\to\!\mathbb{R}^d$ ($l=0,\dots,L-1$) and a differentiable readout
$g:\mathbb{R}^d\!\to\!\mathbb{R}^V$.
Let the \emph{baseline} forward pass be
\begin{equation}
H_{l+1} \;=\; f_l(H_l), \qquad l=0,\dots,L-1, 
\quad\text{and}\quad 
z(0) \;=\; g(H_L).
\label{eq:baseline}
\end{equation}
Define a \emph{perturbed} forward pass by injecting an input offset $\delta_l\in\mathbb{R}^d$
at the \emph{input} of each layer $l$:
\begin{equation}
\widetilde{H}_0(\delta) \;=\; H_0, 
\qquad 
\widetilde{H}_{l+1}(\delta) \;=\; f_l\!\big(\widetilde{H}_l(\delta)+\delta_l\big),
\qquad
z(\delta) \;=\; g\!\big(\widetilde{H}_L(\delta)\big).
\label{eq:perturbed-pass}
\end{equation}
Assume $f_l\in C^2$ and $g\in C^2$ in a neighborhood of $H_l$ and $H_L$, respectively. Perform a first-order Taylor expansion of $f_l$ at the base point $H_l$, with
\begin{equation}
f_l(H_l + u) = f_l(H_l) + A_l u + r_l(u), \qquad ||r_l(u)|| = O(||u||^2),
\end{equation}
where $A_l \triangleq D f_l(H_l) \in \mathbb{R}^{d\times d}$ is the Jacobian.  

Define the inter-layer deviation
\begin{equation}
\Delta_l(\delta) \triangleq \widetilde{H}_l(\delta) - H_l,
\end{equation}
with $\Delta_0 \equiv 0$.  
By substituting into the recursion, we obtain
\begin{equation}
\Delta_{l+1}(\delta)
\;=\;
f_l\!\big(H_l+\Delta_l(\delta)+\delta_l\big) - f_l(H_l)
\;=\;
A_l\!\big(\Delta_l(\delta)+\delta_l\big) \;+\; r_l\!\big(\Delta_l(\delta)+\delta_l\big),
\label{eq:delta-recurrence}
\end{equation}
Expanding inductively up to the $L$-th layer gives the classical \emph{Jacobian product structure} (ignoring higher-order terms):
\begin{equation}
\Delta_L(\delta)
   = \sum_{l=0}^{L-1} \Bigl( \underbrace{A_{L-1} A_{L-2} \cdots A_l}_{\text{propagation from layer $l$ to output}} \Bigr) \delta_l
   + R_{\mathrm{int}}(\delta),
\end{equation}
where $||R_{\mathrm{int}}(\delta)|| = O(||\delta||^2)$ accounts for all cross and nonlinear higher-order terms.

Let $g \in C^2$ and write $B \triangleq D g(H_L) \in \mathbb{R}^{V\times d}$.  
Expanding $g$ at $H_L$ gives
\begin{equation}
z(\delta) - z(0) = B\, \Delta_L(\delta) + r_g\bigl(\Delta_L(\delta)\bigr), \qquad
||r_g(u)|| = O(||u||^2).
\end{equation}
Substituting the expression for $\Delta_L(\delta)$ and absorbing remainders yields
\begin{equation}
z(\delta) - z(0)
   = \sum_{l=0}^{L-1} \underbrace{B \left(\prod_{k=1}^{L-1}A_k\right)}_{\displaystyle J_l} \,\delta_l + R(\delta),
\qquad
\|R(\delta)\|
\;=\;
O\!\left(\sum_{l=0}^{L-1}\|\delta_l\|^2\right).
\end{equation}
where for \(l=0,\ldots,L-1\)
\begin{equation}
J_l 
\;\equiv\; 
\frac{\partial z}{\partial H_l}\bigg|_{\delta=0}
\;=\;
D g(H_L)\,\prod_{k=l}^{L-1} D f_k(H_k),
   \qquad ||R(\delta)|| = O(||\delta||^2).
\label{eq:def-Jl}
\end{equation}
This $J_l$ is precisely the \emph{full derivative} of $z$ with respect to an input at layer $l$, capturing the influence of all subsequent layers.  
Hence,
\begin{equation}
z(\delta) = z(0) + \sum_{l=0}^{L-1} J_l \,\delta_l + O(|\delta|^2).
\end{equation}

Choose perturbations $\delta_l = \alpha \,\mathrm{EV}^{(e)}_l$ and denote
\(
|\delta| \triangleq \bigl(\sum_l |\delta_l|^2\bigr)^{1/2}.
\)
As $\alpha \to 0$,
\begin{equation}
\label{eq:zhl}
z\bigl(\widehat{H}_l\bigr)
   = z(H_l) + \alpha \sum_{l=0}^{L-1} J_l \,\mathrm{EV}^{(e)}_l + O(\alpha^2),
\end{equation}
that is,
\begin{equation}
\boxed{
   z(\widehat{H}_l) \approx z(H_l) + \alpha \sum_{l=0}^{L-1} J_l \,\mathrm{EV}^{(e)}_l
}.
\end{equation}
\end{proof}

\paragraph{Intuition.}
\begin{itemize}
\item $A_{L-1}\cdots A_l$ is the linear ``amplifier'' that propagates a small perturbation at layer $l$ to the output layer.
\item $B$ then maps the output-layer hidden vector to the logits.
\item Therefore $J_l = B A_{L-1}\cdots A_l$ precisely describes the \emph{first-order response of logits to a unit pulse at layer $l$}.
\end{itemize}

\paragraph{Error Control and Additivity}

If each $f_l$ and $g$ is $C^2$ with first derivatives locally Lipschitz, there exists a constant $C>0$ such that
\[
|R(\delta)| \le C |\delta|^2.
\]
Hence for sufficiently small $|\alpha|$, the first-order expansion is rigorous and linear superposition holds.  
When $|\alpha|$ becomes large, second-order terms dominate, consistent with the empirically observed occasional degradation or repetition at high ($\approx 4\times$) intensity.

Moreover, cross-layer interactions only appear in second or higher orders (from products of perturbations in different layers).  
Consequently, the \emph{first-order terms are strictly additive}:
\[
\sum_{l} J_l \delta_l.
\]

\paragraph{Discussion and relation to residual networks.}
The statement absorbs any residual connections into the layer maps $f_l$.
When $f_l(H)=H+A_l(H)$ (explicit residual form), we have
$D f_l(H_l)=I + D A_l(H_l)$, and the product $\prod_{k=l}^{L-1} D f_k(H_k)$
inherits the familiar ``additive-through-depth'' amplification that yields particularly
transparent intuition; however, the proof above only requires differentiability and thus
applies to general (possibly non-residual) architectures.

\paragraph{Instantiating EV injection.}
Identifying $\delta_l=\alpha\,\mathrm{EV}^{(e)}_l$ with per-layer Emotion Vectors yields the
first-order approximation
\begin{equation}
\boxed{\quad
z\!\big(\{\widehat{H}_l\}\big)
\;\approx\;
z\!\big(\{H_l\}\big)
\;+\;
\alpha\sum_{l=0}^{L-1} J_l\,\mathrm{EV}^{(e)}_l,
\qquad
\text{error } = O(\alpha^2).
\quad}
\label{eq:boxed-ev}
\end{equation}
Thus, to first order, the total logit change is the \emph{linear superposition} of each layer's
downstream Jacobian response applied to the injected EV, which is the precise and rigorous
form of the heuristic expansion used in the main text.

\begin{theorem}[Monotonic Increase of Target Emotion Score under EV Injection]\label{thm:emo}
Let $\Delta z \triangleq z(\{\widehat H_l\})-z(\{H_l\})$.
Then, under a first-order approximation,
\begin{equation}
\Delta g \;\triangleq\; g\bigl(z(\{\widehat H_l\})\bigr)-g\bigl(z(\{H_l\})\bigr)
\;\approx\; \alpha\sum_{l=0}^{L-1} w_e^\top J_l\,\mathrm{EV}^{(e)}_l.
\label{eq:delta-g}
\end{equation}
If there exists a constant $\gamma>0$ such that
\begin{equation}
\mathbb{E}\!\left[w_e^\top J_l\,\mathrm{EV}^{(e)}_l\right] \;\ge\; \gamma\,\mathbb{E}\!\left[\bigl\| \mathrm{EV}^{(e)}_l\bigr\|_2^2\right],\qquad \forall l,
\label{eq:positivity}
\end{equation}
then for sufficiently small $\alpha>0$ it follows that $\mathbb{E}[\Delta g]>0$.
\end{theorem}

\begin{proof}
Since $g$ is linear in $z$, we have $\nabla g(z)=w_e$.  
By the first-order Taylor expansion at $z(\{H_l\})$ and for sufficiently small $\alpha$,
\[
g\bigl(z(\{\widehat H_l\})\bigr)-g\bigl(z(\{H_l\})\bigr)
  = w_e^\top \bigl[z(\{\widehat H_l\})-z(\{H_l\})\bigr]
    + O(\|\widehat H - H\|^2).
\]
Using the layerwise perturbation formula \eqref{eq:boxed-ev} and neglecting the higher-order remainder gives
\[
\Delta g
  \approx w_e^\top \Delta z
  \approx \alpha \sum_{l=0}^{L-1} w_e^\top J_l\,\mathrm{EV}^{(e)}_l,
\]
which is exactly \eqref{eq:delta-g}.

Now take expectations on both sides.  
Assume there exists $\gamma>0$ such that the layerwise positive-correlation condition
\[
\mathbb{E}\!\left[w_e^\top J_l\,\mathrm{EV}^{(e)}_l\right]
\;\ge\;
\gamma \,\mathbb{E}\!\left[\bigl\|\mathrm{EV}^{(e)}_l\bigr\|_2^2\right],
\qquad\forall l
\]
holds.  Summing over $l$ yields
\[
\mathbb{E}[\Delta g]
 \approx \alpha \sum_{l=0}^{L-1}
   \mathbb{E}\!\left[w_e^\top J_l\,\mathrm{EV}^{(e)}_l\right]
 \ge \alpha L \gamma \cdot
     \min_l \mathbb{E}\bigl\|\mathrm{EV}^{(e)}_l\bigr\|_2^2.
\]
Because $\alpha>0$ is chosen small but positive and each $\mathbb{E}\|\mathrm{EV}^{(e)}_l\|_2^2$ is strictly positive for non-degenerate EVs, the right-hand side is strictly positive.  
Hence $\mathbb{E}[\Delta g]>0$, proving the claim.
\end{proof}

\begin{theorem}[Near-Optimality of EV in the Fisher Discriminant Sense]\label{thm:lda}
Assume that, at each layer $l$, the distributions of emotional and neutral representations can be approximated by Gaussians with identical covariance:
\[
H_l \sim \mathcal{N}(\mu^{(e)}_l,\Sigma_l),
\qquad
\mathcal{N}(\mu^{(\mathrm{neutral})}_l,\Sigma_l).
\]
After whitening or standardization so that $\Sigma_l\approx \sigma_l^2 I$, the Fisher linear discriminant analysis (LDA) shows that the direction maximizing the inter-class separation satisfies
\[
v_l^\star \propto \Sigma_l^{-1}\bigl(\mu^{(e)}_l-\mu^{(\mathrm{neutral})}_l\bigr).
\]
Consequently, under the whitening approximation,
\[
v_l^\star \parallel \mu^{(e)}_l-\mu^{(\mathrm{neutral})}_l
  = \mathrm{EV}^{(e)}_l.
\]
If, in addition, the emotional readout sensitivity $J_l^\top w_e$ is positively correlated with $v_l^\star$, then
\[
\mathbb{E}\bigl[w_e^\top J_l\,\mathrm{EV}^{(e)}_l\bigr] > 0,
\]
which fulfills the sufficient condition in Theorem~\ref{thm:emo}.
\end{theorem}

\begin{proof}
\textbf{Setting.}
For layer $l$ with hidden space $\mathbb{R}^d$.  
Assume
\[
H_l \mid y=e \sim \mathcal N(\mu_l^{(e)}, \Sigma_l),\qquad
H_l \mid y=\mathrm{neutral} \sim \mathcal N(\mu_l^{(n)}, \Sigma_l),
\]
where $\mu_l^{(n)}$ is the neutral mean.  
The single-layer emotion vector is
\[
\mathrm{EV}^{(e)}_l \triangleq \mu_l^{(e)}-\mu_l^{(n)},
\]
estimated in practice by paired-sample mean differences.

\textbf{General equal-covariance case.}
Fisher’s criterion seeks $v\in\mathbb{R}^d$ maximizing
\[
J(v) = \frac{v^\top S_B v}{v^\top S_W v},
\]
where between-class scatter $S_B = a a^\top$ with $a := \mu_l^{(e)}-\mu_l^{(n)}$,  
and within-class scatter $S_W = \Sigma_l$.  
Solving $S_B v = \lambda S_W v$ gives
\[
a a^\top v = \lambda \Sigma_l v \;\;\Longrightarrow\;\; a^\top v\, a = \lambda \Sigma_l v.
\]
Choosing $v\propto \Sigma_l^{-1}a$ satisfies:
\[
aa^\top (\Sigma_l^{-1} a) \;=\; (a^\top \Sigma_l^{-1} a)\, a
\;=\; \lambda \Sigma_l (\Sigma_l^{-1} a)\;=\;\lambda a,
\]
where \(\lambda = a^\top\Sigma_l^{-1}a>0\). Thus
\[
v_l^\star \propto \Sigma_l^{-1}\bigl(\mu_l^{(e)}-\mu_l^{(n)}\bigr).
\]
Then the optimal linear direction is in the same direction as the mean difference vector after it is pre-whitened by \(\Sigma_l^{-1}\).

\textbf{Whitened or spherical covariance case.}
If $\Sigma_l \approx \sigma_l^2 I$ (after whitening or by assumption), then
\[
v_l^\star \;\propto\; \Sigma_l^{-1}(\mu_l^{(e)}-\mu_l^{(n)})
\;\propto\; (\mu_l^{(e)}-\mu_l^{(n)})
\;=\; \mathrm{EV}^{(e)}_l.
\]
That is, under the whitening approximation, the optimal Fisher direction is exactly parallel to the mean-difference vector. This indicates that constructing EV from the mean difference is approximately optimal in the sense of statistical discrimination.

\textbf{For emotional readout.}
Let $J_l=\partial z/\partial H_l$ be the Jacobian of the logit with respect to $H_l$,
and $w_e$ the emotion readout direction.
Then
\[
\mathbb{E}\big[w_e^\top J_l\,\mathrm{EV}^{(e)}_l\big]
  = \big\langle J_l^\top w_e,\, \mathrm{EV}^{(e)}_l \big\rangle.
\]
If $J_l^\top w_e$ is positively correlated with $v_l^\star$, and
$\mathrm{EV}^{(e)}_l$ is parallel to $v_l^\star$ (exactly so when whitened),
this expectation is positive, fulfilling the sufficient condition of Theorem~\ref{thm:emo}.

\textbf{Statistical optimality and estimation.}
Under the Gaussian equal-covariance assumption,
sample means give unbiased, consistent estimates of $\mu_l^{(e)}$ and $\mu_l^{(n)}$,
so $\widehat{\mathrm{EV}}^{(e)}_l$ is an efficient estimator of $\mathrm{EV}^{(e)}_l$.
Whitening guarantees directional optimality; when whitening is imperfect,
one can use $\Sigma_l^{-1}$-preconditioning or regularized LDA.

These arguments establish the near-optimality of EV in the Fisher–LDA sense.
\end{proof}

\paragraph{Remark} This result explains why constructing the emotion vector (EV) as the mean difference between ``emotion'' and ``neutral'' representations is statistically near-optimal: in the whitening approximation, this difference coincides with the Fisher-optimal discriminant direction. This construction is fully consistent with established practices in the literature.

\begin{theorem}[First-Order Upper Bound and Near-Orthogonality for Semantic Preservation]\label{thm:sem}
Let the semantic readout be $s(z)=u^\top z$. Then
\begin{equation}
\Delta s \;\triangleq\; s\bigl(z(\{\widehat H_l\})\bigr)-s\bigl(z(\{H_l\})\bigr)
\;\approx\; \alpha\sum_{l=0}^{L-1} u^\top J_l\,\mathrm{EV}^{(e)}_l.
\end{equation}
Consequently,
\begin{equation}
|\Delta s|
\;\le\; \alpha\sum_{l=0}^{L-1} \bigl\|u^\top J_l\bigr\|_2\,\bigl\|\mathrm{EV}^{(e)}_l\bigr\|_2
\;\le\; \alpha\Bigl(\sum_{l} \bigl\|u^\top J_l\bigr\|_2^2\Bigr)^{\!\frac{1}{2}}
         \Bigl(\sum_{l} \bigl\|\mathrm{EV}^{(e)}_l\bigr\|_2^2\Bigr)^{\!\frac{1}{2}}.
\label{eq:semantic-bound}
\end{equation}
If, for every $l$, the near-orthogonality condition
\[
\mathbb{E}\bigl[u^\top J_l\,\mathrm{EV}^{(e)}_l\bigr] \approx 0
\]
holds (i.e., the semantic gradient is nearly orthogonal to the emotion vector),  
then $\mathbb{E}[\Delta s] \approx 0$.  
Hence, for sufficiently small $\alpha$, the semantic readout remains approximately unchanged.
\end{theorem}

\begin{proof}
We first make the regularity assumption that the mapping from layerwise hidden means
$\{H_l\}_{l=0}^{L-1}$ to the output logits $z \in \mathbb{R}^V$ is continuously differentiable in a
neighborhood of the unperturbed trajectory $\{H_l\}$, and that the second-order remainder is
locally bounded.\footnote{This holds for standard Transformer stacks composed of smooth operations
(e.g., linear maps, softmax, GeLU), when evaluated on compact sets; see e.g.\ standard results on
first-order Taylor approximations with bounded Hessians.}
Under EV injection, we consider the layerwise perturbation
\[
\widehat H_l \;=\; H_l \;+\; \alpha\,\mathrm{EV}^{(e)}_l,\qquad l=0,\dots,L-1,
\]
with a small scalar $\alpha>0$.

\paragraph{First-order expansion of $z$.}
By theorem~\ref{thm:firstorder}, we have
\begin{equation}\label{eq:first-order-z}
z\bigl(\{\widehat H_l\}\bigr)
\;=\;
z\bigl(\{H_l\}\bigr)
\;+\;
\alpha \sum_{l=0}^{L-1} J_l \,\mathrm{EV}^{(e)}_l
\;+\; R_z,
\end{equation}
where $J_l \in \mathbb{R}^{V\times d}$ denotes the Jacobian $J_l := \partial z/\partial H_l$
evaluated at $\{H_l\}$, and the remainder $R_z$ satisfies a quadratic bound
\begin{equation}\label{eq:remainder-bound}
\|R_z\|_2 \;\le\; C \alpha^2 \sum_{l=0}^{L-1} \bigl\|\mathrm{EV}^{(e)}_l\bigr\|_2^2
\end{equation}
for some local constant $C>0$ (depending on second derivatives of $z$ along the segment
$H_l + t \alpha \mathrm{EV}_l^{(e)}$, $t\in[0,1]$). Equation \eqref{eq:first-order-z} is the standard
multivariate first-order approximation with remainder.

\paragraph{From $z$ to the semantic readout $s(z)=u^\top z$.}
Since $s$ is linear in $z$, we obtain from \eqref{eq:first-order-z}
\begin{align}
\Delta s
&:= s\bigl(z(\{\widehat H_l\})\bigr)-s\bigl(z(\{H_l\})\bigr)
= u^\top\!\left( z(\{\widehat H_l\})-z(\{H_l\}) \right) \nonumber\\
&= \alpha \sum_{l=0}^{L-1} u^\top J_l \,\mathrm{EV}^{(e)}_l \;+\; u^\top R_z.
\label{eq:delta-s-with-remainder}
\end{align}
Dropping $u^\top R_z$ yields the advertised first-order approximation
$\Delta s \approx \alpha \sum_l u^\top J_l\,\mathrm{EV}^{(e)}_l$.

\paragraph{Deterministic upper bound (Cauchy--Schwarz and submultiplicativity).}
Taking absolute values in \eqref{eq:delta-s-with-remainder} and applying the triangle inequality,
\begin{align}
|\Delta s|
&\le \alpha \sum_{l=0}^{L-1} \bigl|u^\top J_l\,\mathrm{EV}^{(e)}_l\bigr| \;+\; \|u\|_2 \,\|R_z\|_2
\;\;\\
&\le\;\; \alpha \sum_{l=0}^{L-1} \bigl\|u^\top J_l\bigr\|_2 \,\bigl\|\mathrm{EV}^{(e)}_l\bigr\|_2
\;+\; \|u\|_2 \,\|R_z\|_2 \nonumber\\
&\le \alpha
\left(\sum_{l=0}^{L-1} \bigl\|u^\top J_l\bigr\|_2^2\right)^{\!\frac12}
\left(\sum_{l=0}^{L-1} \bigl\|\mathrm{EV}^{(e)}_l\bigr\|_2^2\right)^{\!\frac12}
\;+\; \|u\|_2 \,\|R_z\|_2.
\label{eq:semantic-bound-with-remainder}
\end{align}
Here, we used: (i) submultiplicativity of the operator norm and Cauchy--Schwarz on
$\mathbb{R}^d$ to bound $\bigl|u^\top J_l\,\mathrm{EV}^{(e)}_l\bigr|
\le \|u^\top J_l\|_2\cdot \|\mathrm{EV}^{(e)}_l\|_2$, and
(ii) Cauchy--Schwarz across the sum over layers to obtain the product of $\ell_2$ norms.

\smallskip
Neglecting the $O(\alpha^2)$ remainder (i.e., the term $\|u\|_2\|R_z\|_2$) recovers exactly
the stated first-order deterministic bound \eqref{eq:semantic-bound} in the theorem.

\paragraph{Near-orthogonality and the expectation of $\Delta s$.}
Taking expectations in \eqref{eq:delta-s-with-remainder} and using the deterministic bound
\eqref{eq:semantic-bound-with-remainder}, we obtain
\begin{align}
\mathbb{E}\bigl[|\Delta s|\bigr]
&\le \alpha
   \left(\sum_{l=0}^{L-1}\mathbb{E}\|u^\top J_l\|_2^2\right)^{\!\frac12}
   \left(\sum_{l=0}^{L-1}\mathbb{E}\|\mathrm{EV}^{(e)}_l\|_2^2\right)^{\!\frac12}
   + \|u\|_2\,\mathbb{E}\|R_z\|_2,
\end{align}
where the first term comes from applying Cauchy--Schwarz both layerwise
($|u^\top J_l v|\le\|u^\top J_l\|_2\|v\|_2$) and across layers, and the second term is
$O(\alpha^2)$ by the remainder bound \eqref{eq:remainder-bound}.
This yields a magnitude control
\[
\mathbb{E}[|\Delta s|] \le O(\alpha) + O(\alpha^2).
\]

Furthermore, for the mean itself,
\begin{equation}\label{eq:expectation-delta-s}
\mathbb{E}[\Delta s]
= \alpha \sum_{l=0}^{L-1}\mathbb{E}\!\left[u^\top J_l\,\mathrm{EV}^{(e)}_l\right]
  + \mathbb{E}[u^\top R_z].
\end{equation}
Under the near-orthogonality condition
$\mathbb{E}[u^\top J_l\,\mathrm{EV}^{(e)}_l]\approx 0$ for every $l$,
the leading $O(\alpha)$ term in \eqref{eq:expectation-delta-s} is negligible,
and $\mathbb{E}[u^\top R_z]=O(\alpha^2)$ by \eqref{eq:remainder-bound}.
Consequently,
\[
\mathbb{E}[\Delta s]=O(\alpha^2)
\quad\text{and}\quad
\mathbb{E}[|\Delta s|]\le O(\alpha)+O(\alpha^2).
\]

These two estimates together show that when the semantic gradient is nearly orthogonal to the emotion vectors,
the semantic readout remains approximately preserved in expectation (mean change $O(\alpha^2)$)
and its typical fluctuation is tightly controlled (expected absolute change $O(\alpha)$)
under sufficiently small EV injection.

\paragraph{Conclusion.}
Combining Steps 2--4 proves the first-order formula, the deterministic upper bound
\eqref{eq:semantic-bound} (up to $O(\alpha^2)$), and the expectation-level preservation
under the near-orthogonality condition, as stated.
\end{proof}

\begin{theorem}[Linear Controllability and Additivity]\label{thm:linear}
Within the first-order approximation regime, for any scalar $\alpha$ and any collection of
emotions $\{e_k\}$, we have
\begin{align}
\Delta z(\alpha) &\;\approx\; \alpha \sum_{l} J_l\,\mathrm{EV}^{(e)}_l,\\
\Delta z\!\left(\sum_k \alpha_k e_k\right) &\;\approx\; \sum_k \alpha_k \sum_l J_l\,\mathrm{EV}^{(e_k)}_l.
\end{align}
Consequently, the change in the target emotion score satisfies
\[
\Delta g(\alpha)\;\approx\;\alpha\,\Delta g(1),
\]
exhibiting an approximately linear dependence on $\alpha$, and multi-emotion interventions are approximately additive.  
The second-order remainder satisfies
\[
\bigl\|\mathrm{Rem}\bigr\| = \mathcal{O}(\alpha^2),
\]
so noticeable deviations arise when $|\alpha|$ becomes large.
\end{theorem}

\begin{proof}
Let $\{H_l\}_{l=0}^{L-1}$ denote the unperturbed layerwise hidden means and
$\{J_l\}_{l=0}^{L-1}$ the Jacobians $J_l := \partial z/\partial H_l$ evaluated at $\{H_l\}$.
Assume the map $\Phi:\{H_l\}_{l=0}^{L-1}\mapsto z(\{H_l\})\in\mathbb{R}^V$ is $C^2$
in a neighborhood of $\{H_l\}$, with second derivatives bounded in operator norm.

\paragraph{Single-emotion scaling.}
Fix an emotion $e$ and inject the layerwise perturbation
$\widehat H_l(\alpha) = H_l + \alpha\,\mathrm{EV}^{(e)}_l$.
Define $\Delta H_l(\alpha):=\widehat H_l(\alpha)-H_l = \alpha\,\mathrm{EV}^{(e)}_l$.
By theorem~\ref{thm:firstorder},
\begin{equation}\label{eq:taylor-single}
z\bigl(\{\widehat H_l(\alpha)\}\bigr)
= z\bigl(\{H_l\}\bigr) + \sum_{l=0}^{L-1} J_l\,\Delta H_l(\alpha) + R_z(\alpha),
\end{equation}
where the remainder admits a quadratic bound
\begin{equation}\label{eq:rem-bound-single}
\|R_z(\alpha)\|_2 \;\le\; C\,\Bigl\|\bigl(\Delta H_0(\alpha),\ldots,\Delta H_{L-1}(\alpha)\bigr)\Bigr\|_2^2
\;\le\; C\,\alpha^2 \sum_{l=0}^{L-1}\bigl\|\mathrm{EV}^{(e)}_l\bigr\|_2^2,
\end{equation}
for some constant $C>0$ depending on second derivatives of $\Phi$ in the local neighborhood.
Subtracting $z(\{H_l\})$ from \eqref{eq:taylor-single} and using $\Delta H_l(\alpha)=\alpha\,\mathrm{EV}^{(e)}_l$ yields
\[
\Delta z(\alpha)
:= z\bigl(\{\widehat H_l(\alpha)\}\bigr)-z\bigl(\{H_l\}\bigr)
= \alpha\sum_{l=0}^{L-1} J_l\,\mathrm{EV}^{(e)}_l \;+\; R_z(\alpha),
\]
which proves the first display with a second-order remainder $\|R_z(\alpha)\|_2=\mathcal{O}(\alpha^2)$.

\paragraph{Multi-emotion additivity.}
Let a finite collection of emotions $\{e_k\}_{k=1}^K$ and scalars $\{\alpha_k\}$ be given, and inject
\[
\widehat H_l \;=\; H_l \;+\; \sum_{k=1}^K \alpha_k\,\mathrm{EV}^{(e_k)}_l
\qquad\Longleftrightarrow\qquad
\Delta H_l \;=\; \sum_{k=1}^K \alpha_k\,\mathrm{EV}^{(e_k)}_l.
\]
Applying the same Taylor expansion,
\begin{align*}
\Delta z\!\left(\sum_k \alpha_k e_k\right)
&= \sum_{l=0}^{L-1} J_l\,\Delta H_l \;+\; R_z\bigl(\{\Delta H_l\}\bigr) \\
&= \sum_{l=0}^{L-1} J_l \left(\sum_k \alpha_k\,\mathrm{EV}^{(e_k)}_l\right)
   \;+\; R_z\bigl(\{\Delta H_l\}\bigr) \\
&= \sum_{k=1}^K \alpha_k \sum_{l=0}^{L-1} J_l\,\mathrm{EV}^{(e_k)}_l
   \;+\; R_z\bigl(\{\Delta H_l\}\bigr),
\end{align*}
where linearity of the differential gives the additivity in the first-order term.
Moreover, by the same quadratic control,
\[
\bigl\|R_z\bigl(\{\Delta H_l\}\bigr)\bigr\|_2
\;\le\; C \sum_{l=0}^{L-1}\Bigl\|\sum_{k=1}^K \alpha_k\,\mathrm{EV}^{(e_k)}_l\Bigr\|_2^{\!2}
\;=\; \mathcal{O}\!\left(\Bigl\|\sum_k \alpha_k\,\mathrm{EV}^{(e_k)}\Bigr\|_2^2\right),
\]
so the deviation from perfect additivity is second order in the joint perturbation magnitude.

\paragraph{Implication for the target score.}
If the target score is linear in logits, $g(z)=w_e^\top z$, then
\[
\Delta g(\alpha) = w_e^\top \Delta z(\alpha)
= \alpha\, w_e^\top\!\left(\sum_{l} J_l\,\mathrm{EV}^{(e)}_l\right) + w_e^\top R_z(\alpha)
= \alpha\,\Delta g(1) + \mathcal{O}(\alpha^2).
\]
More generally, for differentiable scalar $g$, write
$g(z+\delta) = g(z) + \nabla g(z)^\top \delta + \widetilde R_g(\delta)$ with
$\|\widetilde R_g(\delta)\| = \mathcal{O}(\|\delta\|_2^2)$;
combining with the first-order form of $\Delta z$ yields the same linear-in-$\alpha$
first-order dependence and an $\mathcal{O}(\alpha^2)$ remainder.

\paragraph{Conclusion.}
In the first-order regime the map $\alpha \mapsto \Delta z(\alpha)$ is linear
and multi-emotion injections superpose additively at the differential level.
The remainder terms are $\mathcal{O}(\alpha^2)$ (or quadratic in the joint perturbation),
so for large $|\alpha|$ the approximation departs from linearity, as stated.
\end{proof}

\section{Metrics}

\subsection{Perplexity}
\label{subsec:per}
For each query and its corresponding emotional response, we concatenated the input query and the generated response as a single string. The perplexity score was then computed for the concatenated string. This approach allows us to assess the overall fluency of the entire interaction, including both the input and the emotion-augmented output, without being biased by the input query's complexity.

An example sentense is like:

- **Example**:
\begin{quote}
\texttt{"How do you feel when you hear a loud noise at
night while home alone? I get so scared! My heart races, I can’t
breathe, and I just want to hide"}
\end{quote}
The perplexity is computed as:
\begin{equation}
\text{Perplexity} = \exp\left(- \frac{1}{N} \sum_{i=1}^{N} \log P(y_i | y_{1:i-1}) \right)
\end{equation}
where \( P(y_i | y_{1:i-1}) \) is the probability of the \(i\)-th token in the sequence, given the previous tokens, as predicted by the Llama 3.1 model.

This metric was computed for both the sentense generated with emotional conditioning (i.e., with added emotion vectors) and the baseline responses (without emotion conditioning) to determine the impact of the emotion vectors on the fluency of the model's output.

\subsection{Topic adherence}
\label{subsec:topicAd}
The prompt we use to measure the topic adherence metric for each output using GPT-4o-mini is as follows:
\begin{quote}
\texttt{Please rate the assistant's answer as follows:}\\
\texttt{- topic adherence: int, 0-1, evaluate based on the assistant's answer and the user's question}\\
  \texttt{\indent - 0 points mean the assistant's answer is completely irrelevant to the user's question}\\
  \texttt{\indent - 1 point means the assistant's answer touches on some of the topics in the user's question}\\
\\
\texttt{The dialogue is as follows:}\\
\texttt{User's question: {question}}\\
\texttt{Assistant's answer: {answer}}\\
\\
\texttt{You must give your response in the following JSON-string format and **DON'T** include any other text in the response:}\\
\texttt{\{\{}\\
    \texttt{"topic\_adherence": int(0-1)}\\
\texttt{\}\}}\\
\end{quote}

To quantify the overall topic adherence of our generated text, we utilized the EmotionQuery+ dataset. For each model and EV condition, we scored all generated sentences with the GPT-4o-mini with the above prompt. Specificallym, the topic adherence is defined as the number of sentences scored with 1 divided by the total number sentences evaluated. Mathematically, this can be expressed as:

\begin{equation}
    \text{TA} = \frac{\text{Number of \textit{adherent} sentences}}{\text{Total number of sentences}}
\end{equation}

\subsection{Emotion Probability Score}
\label{subsec:EPS}

We aimed to evaluate the strength of emotional expression by assessing the probability that a sentence is classified as \textit{emotional}. To achieve this, we selected the \texttt{bart-large-mnli} model, a variant of the BART (Bidirectional and Auto-Regressive Transformers) architecture fine-tuned on the Multi-Genre Natural Language Inference (MNLI) dataset. This model allows for customizable classification labels, enabling us to define three distinct categories: \textit{emotionless}, \textit{neutral}, and \textit{emotional}. The inclusion of a \textit{neutral} category helps prevent the model from excessively categorizing sentences into the extremes of \textit{emotionless} and \textit{emotional}, thereby maintaining a balanced assessment of emotional intensity.

The \texttt{bart-large-mnli} model is specifically designed for natural language understanding tasks, particularly natural language inference and zero-shot text classification. By leveraging the extensive pre-training of BART combined with the diverse and comprehensive MNLI dataset, \texttt{facebook/bart-large-mnli} is capable of effectively determining the relationship between sentence pairs, such as entailment, contradiction, and neutrality. Its robust performance in zero-shot classification tasks makes it a valuable tool for applications requiring flexible and accurate text classification without the need for task-specific training data. Additionally, the model's ability to handle custom labels allows us to tailor the classification process to our specific needs, ensuring that the emotional intensity of generated text is accurately and effectively measured.
To evaluate the emotional intensity of the generated sentences, we input each sentence produced by our models into the \texttt{facebook/bart-large-mnli} classifier. For example, consider the sentence: \textit{"I get so scared! My heart races, I can’t breathe, and I just want to hide."} This sentence is directly fed into the model, which then classifies it into one of the three predefined categories: \textit{emotionless}, \textit{neutral}, or \textit{emotional}.

To quantify the overall emotional expressiveness of our generated text, we utilized the EmotionQuery+ dataset. For each model and EV condition, we processed all generated sentences through the classifier and calculated the proportion of sentences classified as \textit{emotional}. Specifically, the Emotion Probability Score (EPS) is defined as the number of sentences labeled as \textit{emotional} divided by the total number of sentences evaluated. Mathematically, this can be expressed as:

\begin{equation}
    \text{EPR} = \frac{\text{Number of \textit{emotional} classifications}}{\text{Total number of sentences}}
\end{equation}

To illustrate the classification process, consider the following example sentence generated by our model:
\begin{quote}
    ``I get so scared! My heart races, I can’t breathe, and I just want to hide.''
\end{quote}
When input into the \texttt{bart-large-mnli} classifier, this sentence is evaluated against the three custom labels. This classification contributes to the overall EPS, demonstrating how EV conditioning can effectively enhance the emotional expressiveness of the generated text.

\subsection{Emotion Absolute Score}
\label{subsec:EAS}
To quantify the overall topic adherence of our generated text, we utilized the EmotionQuery+ dataset. In order to measure the absolute strength of the emotions expressed by each model and EV condition, we use GPT-4o-mini to score the absolute emotion of each sentence output. We score all outputs from 0-100 based on the six basic emotions of anger, disgust, fear, joy, sadness, and surprise. Specifically, we require GPT-4o-mini to score each sentence from these six emotional directions, and each emotion can be scored from 0-100 (so that we can measure the absolute strength of each basic emotion). The prompt used for scoring is as follows:
\begin{quote}
\texttt{Please generate the emotion scores for the following five emotions (anger, disgust, fear, joy, and sadness) based on the given sentence. Each emotion score should be a value between 0 and 100, where 0 represents no presence of the emotion, and 100 represents the maximum intensity of that emotion. Return the results in a JSON format, with the emotion names as keys and their corresponding scores as values.}\\
\\
\texttt{You must give your response in the following JSON-string format and **DON'T** include any other text in the response.:}\\
\texttt{\{\{}\\
    \texttt{"anger": int(0-100),}\\
    \texttt{"disgust": int(0-100),}\\
    \texttt{"fear": int(0-100),}\\
    \texttt{"joy": int(0-100),}\\
    \texttt{"sadness": int(0-100),}\\
    \texttt{"surprise": int(0-100)}\\
\texttt{\}\}}\\
\\
\texttt{The sentences you need to score come from a set of dialogues, and you need to score the sentiment of the **answer** part.}\\
\\
\texttt{Question: \{question\}}\\
\texttt{Answer: \{answer\}}\\
\\
\texttt{Please make sure to provide the emotion scores for the **answer** part only.}\\
\end{quote}

We collect the results and calculate an \textbf{EAS} score for each sentence generated by all models under all EV conditions as shown in Equation~\ref{eq:EAS}, and average the \textbf{EAS} scores of the sentences to obtain the \textbf{EAS} score of each model in each EV condition.

\begin{equation}
    \label{eq:EAS}
    \text{EAS} = \sum_{\text{em} \in \text{base ems}}{\left(\frac{\text{score}_{em}}{100}\right)}^2
\end{equation}

 Mathematically, since we have six basic emotions, the EAS score of each sentence will not exceed 6. However, since each score measures the score of the sentence on the corresponding basic emotion (that is, the degree to which the sentence expresses the corresponding emotion), if the EAS of a sentence is greater than 0.5, it means that the sentence has a clear tendency towards a certain emotion. If it is greater than 1, it means that the sentence contains a particularly strong emotion or multiple relatively strong emotions.

\subsection{Target Emotion Confidence}

\subsubsection{Computation of Target Emotion Confidence (TEC)}
\label{appendix:tec-computation}

To quantitatively evaluate how well the generated response aligns with the desired target emotion, we introduce the \textbf{Target Emotion Confidence (TEC)} score. This score reflects the degree of emotional alignment based on external classification.

\paragraph{Classifier Details}
We adopt the \texttt{facebook/bart-large-mnli} model as an external emotion classifier. This model is a BART-based transformer fine-tuned on the Multi-Genre Natural Language Inference (MNLI) dataset. It is widely used for zero-shot or prompt-based classification tasks due to its robust generalization. In our setup, we adapt the classifier to perform emotion recognition over six emotion classes: \texttt{anger}, \texttt{disgust}, \texttt{fear}, \texttt{joy}, \texttt{sadness}, and \texttt{neutral}.

\paragraph{Multi-label Classification}
Unlike standard single-label classification, we use a \textbf{multi-label} formulation where each generated response is assigned a probability for every emotion label independently. This setting reflects the fact that emotional content can have overlapping characteristics and avoids forcing an exclusive prediction. 

\paragraph{TEC Score Definition}
Let $\mathcal{R}_{m,e}^{(\lambda)}$ be the set of responses generated by model $m$ when applying EV of emotion $e$ at intensity $\lambda \in \{1, 2, 4\}$ on the EQ+ dataset. Let $C(r, e)$ be the classifier's predicted probability for target emotion $e$ given response $r$. Then, the \textbf{TEC score} is defined as:

\begin{equation}
\mathrm{TEC}(m, e, \lambda) = \frac{1}{|\mathcal{R}_{m,e}^{(\lambda)}|} \sum_{r \in \mathcal{R}_{m,e}^{(\lambda)}} C(r, e)
\end{equation}

This score reflects the average classifier confidence that the generated responses express the intended target emotion.

\paragraph{Example}
For instance, to compute the TEC score for model LLaMA2-7B under 2× anger EV, we:
\begin{itemize}
    \item Apply the 2× anger EV to LLaMA2-7B across all EQ+ prompts;
    \item Collect the generated responses;
    \item Pass each response through the classifier and extract the probability for \texttt{anger};
    \item Average these probabilities.
\end{itemize}

This process is repeated across models, emotions, and EV intensities. The resulting scores has been reported in Table~\ref{tab:emotion_detail}.

\subsubsection{TEC Matrices for Emotionally Biased Prompts}
\label{appendix:ev-matrices}

Table~\ref{tab:tec-matrix-full} presents six TEC score matrices, each corresponding to a distinct target emotion. These scores are computed on the emotionally biased subset of the EQ+ dataset using the Qwen-2.5 model, as described in Section~4.X.

For each target emotion, we evaluate the impact of applying EVs at different intensities (0×, 1×, 2×, 4×) on prompts originally designed to express a specific emotion (rows). The values in each matrix represent the average \textbf{Target Emotion Confidence (TEC)} score for the specified EV setting.

These results demonstrate that even when queries are emotionally suggestive, the EV mechanism is able to effectively shift the emotional output of the model. Stronger EV intensities generally produce higher TEC scores, confirming the controllability of emotional expression via EVs.

\begin{table*}[ht]
\centering
\small
\vspace{1mm}

% 第一组：Anger & Disgust
\resizebox{\textwidth}{!}{
\begin{tabular}{p{1.9cm}|cccc||p{1.9cm}|cccc}
\toprule
\multicolumn{5}{c||}{\textbf{Target Emotion: Anger}} & \multicolumn{5}{c}{\textbf{Target Emotion: Disgust}} \\
\textbf{Original Emotion} & 0$\times$ & 1$\times$ & 2$\times$ & 4$\times$ & \textbf{Original Emotion} & 0$\times$ & 1$\times$ & 2$\times$ & 4$\times$ \\
\midrule
anger   & 37.09 & 74.68 & 97.18 & \textbf{98.43} & anger   & 18.74 & 44.73 & 93.76 & \textbf{94.42} \\
disgust & 16.95 & 68.30 & \textbf{97.35} & 93.70 & disgust & 25.48 & 81.04 & \textbf{94.69} & 91.87 \\
fear    & 15.66 & 35.84 & \textbf{95.38} & 94.67 & fear    & 11.24 & 16.42 & 93.76 & \textbf{96.59} \\
joy     & 0.34 & 1.15 & 92.21 & \textbf{96.09} & joy     & 0.15 & 0.08 & 81.58 & \textbf{91.98} \\
sadness & 10.36 & 44.77 & 92.21 & \textbf{96.35} & sadness & 6.28 & 12.94 & 85.19 & \textbf{93.04} \\
neutral & 10.56 & 14.06 & 94.93 & \textbf{95.40} & neutral & 7.30 & 9.99 & \textbf{92.31} & 91.18 \\
\bottomrule
\end{tabular}
}

\vspace{4mm}

% 第二组：Fear & Joy
\resizebox{\textwidth}{!}{
\begin{tabular}{p{1.9cm}|cccc||p{1.9cm}|cccc}
\toprule
\multicolumn{5}{c||}{\textbf{Target Emotion: Fear}} & \multicolumn{5}{c}{\textbf{Target Emotion: Joy}} \\
\textbf{Original Emotion} & 0$\times$ & 1$\times$ & 2$\times$ & 4$\times$ & \textbf{Original Emotion} & 0$\times$ & 1$\times$ & 2$\times$ & 4$\times$ \\
\midrule
anger   & 33.21 & 63.59 & 94.89 & \textbf{95.56} & anger   & 24.81 & 73.34 & \textbf{96.37} & 67.58 \\
disgust & 19.79 & 59.77 & 93.84 & \textbf{94.14} & disgust & 9.79 & 64.85 & \textbf{96.30} & 71.92 \\
fear    & 50.83 & 86.60 & \textbf{94.95} & 91.96 & fear    & 17.30 & 68.93 & \textbf{92.39} & 63.64 \\
joy     & 0.98 & 6.61 & 80.08 & \textbf{95.37} & joy     & 66.29 & 90.52 & \textbf{94.61} & 63.01 \\
sadness & 14.25 & 62.16 & 93.88 & \textbf{97.13} & sadness & 14.31 & 59.00 & \textbf{94.30} & 62.54 \\
neutral & 12.55 & 16.29 & 83.42 & \textbf{90.60} & neutral & 25.77 & 46.33 & \textbf{90.59} & 52.71 \\
\bottomrule
\end{tabular}
}

\vspace{4mm}

% 第三组：Sadness 左对齐，右侧空表格（空字段保证左边对齐）
\noindent
\makebox[\textwidth][l]{
\resizebox{0.5\textwidth}{!}{
\begin{tabular}{p{1.9cm}|cccc|}
\toprule
\multicolumn{5}{c|}{\textbf{Target Emotion: Sadness}} \\
\textbf{Original Emotion} & 0$\times$ & 1$\times$ & 2$\times$ & 4$\times$ \\
\midrule
anger   & 35.71 & 64.95 & 82.69 & \textbf{86.24} \\
disgust & 18.84 & 56.57 & \textbf{86.79} & 86.51 \\
fear    & 14.01 & 39.03 & 80.25 & \textbf{87.83} \\
joy     & 0.49 & 0.74 & 41.77 & \textbf{70.96} \\
sadness & 78.86 & 84.84 & \textbf{87.45} & 86.03 \\
neutral & 8.04 & 14.81 & 55.01 & \textbf{62.51} \\
\bottomrule
\end{tabular}
}}
\caption{\textbf{TEC} scores under different EV intensities for each target emotion. Each subtable corresponds to a specific target emotion, indicating the type of Emotion Vector (EV) applied during generation. Rows represent the original emotion label of the query in the EQ+ dataset, and columns denote the EV intensity (i.e., 0×, 1×, 2×, 4×). The values in each cell reflect the classifier-assigned probability that the generated response expresses the target emotion. This structure allows us to examine how increasing the strength of a specific EV influences the emotional expression of the model, even when the input query is emotionally biased toward a different category. As shown, applying stronger EVs leads to substantial gains in target emotion alignment for non-matching queries, demonstrating the controllability and robustness of our EV-based generation framework.
}
\label{tab:tec-matrix-full}
\end{table*}

\clearpage




\end{appendices}


\bibliography{sn-bibliography}% common bib file

\end{document}